\documentclass{article}

\usepackage[utf8x]{inputenc}
\usepackage{ucs}

\usepackage[cmex10]{amsmath}
\usepackage{amsfonts}
\usepackage{amssymb}
\usepackage{amsthm}
\usepackage{stmaryrd}

\usepackage{color}
\usepackage{natbib}
\usepackage{cite}

	\usepackage{subcaption}
\usepackage{afterpage}
\usepackage{url}

\usepackage{tikz}
\usetikzlibrary{arrows,automata,petri,calc,shapes,positioning}
\usepackage{pgfplots}

\newcommand{\putBibliography}{
	\bibliography{./library,/home/leshyk/PhD/library}
}

\theoremstyle{plain}

\newtheorem{corollary}{Corollary}

\newtheorem{definition}{Definition}
\newtheorem{proposition}{Proposition}
\newtheorem{example}{Example}

\newcommand{\setbegin}{\left\{}
\newcommand{\setend}{\right\}}
\newcommand{\setdlmt}{,}

\newcommand{\seqbegin}{\left<}
\newcommand{\seqend}{\right>}
\newcommand{\seqdlmt}{;}

\newcommand{\f}{\right.}
\newcommand{\s}{\left.}

\let\ifset\iffalse
\newcommand{\set}[1]{
	\let\iftmp=\ifset
	\let\ifset\iftrue
	\setbegin#1\setend
	\let\ifset=\iftmp
}
\newcommand{\seq}[1]{
	\let\iftmp=\ifset
	\let\ifset\iffalse
	\seqbegin#1\seqend
	\let\ifset=\iftmp
}
\newcommand{\+}{
\ifset\setdlmt\else\seqdlmt\fi
}

\newcommand{\pI}{\seq{\el{H_1}{a}\+\el{H_1}{c\+d}\+\el{H_1}{a\+b}\+\el{H_1}{d}}}
\newcommand{\pIRef}{p^1}
\newcommand{\pII}{\seq{\el{H_2}{c\+d}\+\el{H_3}{b\+d}\+\el{H_3}{a\+d}}}
\newcommand{\pIIRef}{p^2}
\newcommand{\pIII}{\seq{\el{H_4}{c\+d}\+\el{H_4}{b}\+\el{H_4}{a}\+\el{H_4}{a\+d}}}
\newcommand{\pIIIRef}{p^3}

\newcommand{\el}[2]{[#1,\set{#2}]}
\newcommand{\elnP}[1]{[#1]}

\newcommand{\ssI}{\seq{\el{CH}{c\+d}\+\el{H_1}{b}\+\el{*}{d}}}
\newcommand{\ssIRef}{ss^1}
\newcommand{\ssII}{\seq{\el{CH}{c,d}\+\el{*}{b}\+\el{*}{d}}}
\newcommand{\ssIIRef}{ss^2}
\newcommand{\ssIII}{\seq{\el{CH}{}\+\el{*}{d}\+\el{*}{a}}}
\newcommand{\ssIIIRef}{ss^3}
\newcommand{\ssIV}{\seq{\el{*}{c,d}\+\el{*}{b}}}
\newcommand{\ssIVRef}{ss^4}
\newcommand{\ssV}{\seq{\el{*}{a}}}
\newcommand{\ssVRef}{ss^5}
\newcommand{\ssVI}{\seq{\el{*}{c\+d}\+\el{CL}{b}\+\el{CL}{a}}}
\newcommand{\ssVIRef}{ss^6}
\newcommand{\ssVII}{\seq{\el{CL}{d}\+\el{CL}{}}}
\newcommand{\ssVIIRef}{ss^7}
\newcommand{\ssVIII}{\seq{\el{CL}{}\+\el{CL}{a\+d}}}
\newcommand{\ssVIIIRef}{ss^8}
\newcommand{\ssIX}{\seq{\el{CH}{c,d}}}
\newcommand{\ssIXRef}{ss^9}
\newcommand{\ssX}{\seq{\el{CL}{b}\+\el{CL}{a}}}
\newcommand{\ssXRef}{ss^{10}}
\newcommand{\ssXI}{\seq{\el{*}{c\+d}\+\el{*}{b}}}
\newcommand{\ssXIRef}{ss^{11}}
\newcommand{\ssXII}{\seq{\el{*}{a}\+\el{*}{d}}}
\newcommand{\ssXIIRef}{ss^{12}}

\newcommand{\extC}[1]{\mathtt{Ext}(#1)}
\newcommand{\intC}[1]{\mathtt{Int}(#1)}
\newcommand{\stab}[1]{\mathtt{Stab}(#1)}
\newcommand{\diff}{\Delta}

\newcommand{\bote}{\bot_E}
\newcommand{\sqcape}{\sqcap_E}
\newcommand{\sqsubseteqe}{\sqsubseteq_E}

\newcommand{\tbl}[2]{\caption{#1}{\begin{center}#2\end{center}}}
\newcommand{\toprule}{\hline\hline}
\newcommand{\colrule}{\hline}
\newcommand{\botrule}{\hline\hline}

\newcommand{\tikznamedpicture}[3][0]{
	\newcommand{#2}[#1]{ \begin{tikzpicture}#3\end{tikzpicture} }
}
\newcommand{\concept}[2]{$\left(#1;#2\right)$}

\tikznamedpicture{\putRETaxonomy}{[
	every node/.style = {draw,circle,font={\tiny}},
	node distance= 2mm and 5mm
]
	\node(h1){$H_1$};
	\node(h2)[right=of h1]{$H_2$};
	\node(h3)[right=of h2]{$H_3$};
	\node(h4)[right=of h3]{$H_4$};
	\node(ch)[above=of $(h1)!0.5!(h2)$]{CH}
		edge(h1)
		edge(h2);
	\node(cl)[above=of $(h3)!0.5!(h4)$]{CL}
		edge(h3)
		edge(h4);
	\node(top)[above=of $(ch)!0.5!(cl)$]{*}
		edge(ch)
		edge(cl);
}

\tikznamedpicture{\putIPSLattice}{[
	every node/.style={draw,rectangle,dotted,font={\tiny}},
	node distance= 0.5cm and 1cm
]
	\node(g1){
		\concept{\set{g_1}}{\seq{[1,3]\+[3,5]\+[2,4]}}
	};
	\node(g2) [right=of g1] {
		\concept{\set{g_2}}{\seq{[5,7],[4,6],[2,5]}}
	};
	\node(g3) [right=of g2] {
		\concept{\set{g_3}}{\seq{[1,9],[2,7],[6,6]}}
	};
	\node(g12)[above=of $(g1)!.5!(g2)$]{
		\concept{\set{g_1\+g_2}}{\seq{[1,7]\+[3,6]\+[2,5]}}
	}
		edge (g1)
		edge (g2);
	\node(g123) [above=of g12 -| g2]{
		\concept{\set{g_1\+g_2\+g_3}}{\seq{[1,9],[2,7],[2,6]}}
	}
		edge (g12)
		edge (g3);
	\node(bottom) [below=of g2] {
		\concept{\emptyset}{*}
	}
		edge (g1)
		edge (g2)
		edge (g3);
}

\tikznamedpicture{\putFCALattice}{[
	every node/.style={draw,rectangle,font={\bf}},
	node distance= 0.5cm and 0.5cm
	]
	\node(bottom){\concept{}{\set{m_1\+m_2\+m_3\+m_4}}};
	\node(g2g4)[above=of bottom] {\concept{g_2\+g_4}{\set{m_3\+m_4}}}
		edge (bottom);
	\node(g1)[left=of g2g4] {\concept{\set{g_1}}{\set{m_1\+m_4}}}
		edge (bottom);
	\node(g3) [right=of g2g4] {\concept{\set{g_3}}{\set{m_2}}}
		edge (bottom);
	\node (g1g2g4) [above=of $(g1)!0.5!(g2g4)$] {\concept{\set{g_1\+g_2\+g_4}}{\set{m_4}}}
		edge(g1)
		edge(g2g4);
	\node(top) [above=of g1g2g4 -| g2g4] {\concept{\set{g_1\+g_3\+g_2\+g_4}}{}}
		edge(g1g2g4)
		edge(g3);
}
\tikznamedpicture{\putSPSLattice}{[
	every node/.style={draw,rectangle,font={\tiny}},
	node distance= 0.5cm and 1.5cm
]
	\node(p2){
		\concept{\set{p^2}}{\pIIRef}
	};
	\node(p1)[left=of p2]{
		\concept{\set{p^1}}{\pIRef}
	};
	\node(p3)[right=of p2]{
		\concept{\set{p^3}}{\pIIIRef}
	};
	\node(p1p2)[above=of p1] {
		\concept{\set{p^1,p^2}}{\ssIIRef,\ssIIIRef}
	}
		edge (p1)
		edge (p2);
	\node(p1p3)[above=of p2] {
		\concept{\set{p^1,p^3}}{\ssXIRef,\ssXIIRef}
	}
		edge (p1)
		edge (p3);
	\node(p2p3)[above=of p3]{
		\concept{\set{p^2,p^3}}{\ssVIRef,\ssVIIRef,\ssVIIIRef}
	}
		edge (p2)
		edge (p3);
		
	\node(top)[above=of p1p3]{
		\concept{\set{p^1,p^2,p^3}}{\ssIVRef,\ssVRef}
	}
		edge (p1p2)
		edge (p1p3)
		edge (p2p3);
		
	\node(bottom)[below=of p2]{
		\concept{\emptyset}{*}
	}
		edge (p1)
		edge (p2)
		edge (p3);
}

\tikznamedpicture{\putSPSLatticeMLPIII}{[
	every node/.style={draw,rectangle,font={\tiny\bf}},
	node distance= 0.5cm and 0.3cm
	]
	\node(p2){
		\concept{\set{p^2}}{\pIIRef}
	};
	\node(p1)[left=of p2]{
		\concept{\set{p^1}}{\pIRef}
	};
	\node(p3)[right=of p2]{
		\concept{\set{p^3}}{\pIIIRef}
	};
	\node(p1p2)[above=of p1] {
		\concept{\set{p^1,p^2}}{\ssIIRef,\ssIIIRef}
	}
	edge (p1)
	edge (p2);
	\node(p2p3)[above=of p3]{
		\concept{\set{p^2,p^3}}{\ssVIRef}
	}
	edge (p2)
	edge (p3);
	
	\node(top)[above=of $(p1p2)!0.5!(p2p3)$]{
		\concept{\set{p^1,p^2,p^3}}{\emptyset}
	}
	edge (p1p2)
	edge (p2p3);
	
	\node(bottom)[below=of p2]{
		\concept{\emptyset}{*}
	}
	edge (p1)
	edge (p2)
	edge (p3);
}

\tikznamedpicture{\putSPSLatticeNoTypeRoot}{[
	every node/.style={draw,rectangle,font={\tiny}},
	node distance= 0.5cm and 0.3cm
	]
	\node(p2){
		\concept{\set{p^2}}{\pIIRef}
	};
	\node(p1)[left=of p2]{
		\concept{\set{p^1}}{\pIRef}
	};
	\node(p3)[right=of p2]{
		\concept{\set{p^3}}{\pIIIRef}
	};
	\node(p1p2)[above=of p1] {
		\concept{\set{p^1,p^2}}{\ssIXRef}
	}
	edge (p1)
	edge (p2);
	\node(p2p3)[above=of p3]{
		\concept{\set{p^2,p^3}}{\ssVIIRef,\ssVIIIRef,\ssXRef}
	}
	edge (p2)
	edge (p3);
	
	\node(top)[above=of $(p1p2)!0.5!(p2p3)$]{
		\concept{\set{p^1,p^2,p^3}}{\emptyset}
	}
	edge (p1p2)
	edge (p2p3);
	
	\node(bottom)[below=of p2]{
		\concept{\emptyset}{*}
	}
	edge (p1)
	edge (p2)
	edge (p3);
}

\tikznamedpicture{\putStabLattice}{[
	every node/.style={draw,rectangle,font={\tiny}},
	node distance= 0.5cm and 0.3cm
	]
	\node(g1){\concept{\set{g_1}}{*}\bf[0.5]};
	\node(g2)[right=of g1]{\concept{\set{g_2}}{*}\bf[0.5]};
	\node(g3)[right=of g2]{\concept{\set{g_3}}{*}\bf[0.5]};
	\node(g4)[right=of g3]{\concept{\set{g_4}}{*}\bf[0.5]};
	\node(g5)[right=of g4]{\concept{\set{g_5}}{*}\bf[0.5]};
	\node(bottom)[below=of g3] {\concept{\emptyset}{*}\bf[1.0]}
	edge(g1)
	edge(g2)
	edge(g3)
	edge(g4)
	edge(g5);
	\node(g1234)[above=of $(g2)!0.5!(g3)$,very thick] {\concept{\bf\set{g_1\+g_2\+g_3\+g_4}}{\bf\set{m_6}}\bf[0.69]}
	edge(g1)
	edge(g2)
	edge(g3)
	edge(g4);
	\node(top)[above=of g1234-|g3]{\concept{\set{g_1\+g_2\+g_3\+g_4\+g_5}}{*}\bf[0.47]}
	edge(g1234)
	edge(g5);
}

\begin{document}

\title{
	On mining complex sequential data by means of FCA and pattern structures
}

\author{
	Aleksey Buzmakov$^{\rm a,c}$$^{\ast}$\thanks{
		$^\ast$Corresponding author. Email: aleksey.buzmakov@inria.fr \vspace{6pt}
	} \and Elias Egho$^{\rm b,1}$\footnote{
		$^1$Elias Egho was in LORIA (Vandoeuvre-les-Nancy, France) when this work was done.
	} \and Nicolas Jay$^{\rm a}$
	\and Sergei O. Kuznetsov$^{\rm c}$
	\and Amedeo Napoli$^{\rm a}$ \and Chedy Raïssi$^{\rm a}$\\
}
\date{
	$^{a}${\em{Orpailleur, LORIA (CNRS -- Inria NGE -- U. de Lorraine), Vandoeuvre-lès-Nancy, France}}; \\
	$^{b}${\em{
		Orange Labs, 
		%2, avenue Pierre Marzin
		Lannion, France
	}} \\
	$^{c}${\em{National Research University Higher School of Economics, Moscow, Russia}}
	%$^{a}${\em{LORIA (CNRS -- Inria NGE -- U. de Lorraine), 615, rue Jardin Botanique, Vandœuvre-lès-Nancy, France}};
	%$^{b}${\em{National Research University Higher School of Economics, 20 Myasnitskay street, Moscow, Russia}}
	%\\\received{v5.1 released February 2014} 
}

\maketitle

\begin{abstract}
	Nowadays data sets are available in very complex and heterogeneous ways. Mining of such data collections is essential to support many real-world applications ranging from healthcare to marketing.  In this work, we focus on the analysis of \emph{``complex''} sequential data by means of interesting sequential patterns. We approach the problem using the elegant mathematical framework of Formal Concept Analysis (FCA) and its extension based on \emph{``pattern structures"}. Pattern structures are used for mining complex data (such as sequences or graphs) and are based on a subsumption operation, which in our case is defined with respect to the partial order on sequences. We show how pattern structures along with projections (i.e., a  data reduction of sequential structures), are able to enumerate more meaningful patterns and increase the computing efficiency of the approach. Finally, we show the applicability of the presented method for discovering and analyzing interesting patient patterns from a French healthcare data set on cancer. The quantitative and qualitative results (with annotations and analysis from a physician) are reported in this use case which is the main motivation for this work.

{\bf Keywords:}
	data mining;
	formal concept analysis;
	pattern structures;
	projections;
	sequences;
	sequential data.

\end{abstract}

%%%%%%%%%%%%%%%%%%%%%%%%%%%%%%%%%%%%%%%%%%%%%%%%%%%%%%%%%%%%%%%%%%%%%%%%%%%

\section{Introduction}\label{sect:introduction}

Sequence data is present and used in many applications.  Mining sequential patterns from sequence data has become an important data mining task.  In the last two decades, the main emphasis has been on developing efficient mining algorithms and effective pattern representations~\citep{Han2000,Pei2001,Yan2003,Ding2009,Raissi2008}. However, one problem with traditional sequential pattern mining algorithms (and generally with all pattern enumeration algorithms) is that they generate a large number of frequent sequences while a few of them are truly relevant. To tackle this challenge, recent studies try to enumerate patterns using some alternative interestingness measures or by sampling representative patterns. A general idea in finding \emph{statistically significant patterns} is to extract patterns whose characteristics for a given measure, such as frequency, strongly deviates from its expected value under a null model, i.e. the value expected by the distribution of all data. In this work, we focus on complementing the statistical approaches with a sound algebraic approach trying to answer the following question: \emph{can we develop a framework for enumerating only relevant patterns based on data lattices and its associated measures?}
%Moreover, in some particular cases, only sequential patterns of a certain type are of interest and should be mined first. In addition, another important drawback of these pattern enumeration algorithms is that they depend on a user selected support threshold, which is usually hard to be properly set by non-experts. \emph{How can one avoid the setting of support threshold while having optimal pattern analysis? Can we develop a framework for taking into account only patterns of required types?}

The above question can be answered by addressing the problem of analyzing sequential data using the framework of Formal Concept Analysis (FCA), a mathematical approach to data analysis \citep{Ganter1999}, and pattern structures,
%\footnote{Pattern structures should not be confused with pattern mining}
an extension of FCA that handles complex data \citep{Ganter2001}. %Sequential data in this framework is processed by taking into account the ordering only between consecutive elements.
To analyze a dataset of ``complex" sequences while avoiding the classical efficiency bottlenecks, we introduce and explain the usage of projections, which are mathematical mappings for defining approximations. Projections for sequences allow one to reduce the computational costs and the volume of enumerated patterns, avoiding the infamous ``pattern flooding''. In addition, we provide and discuss several measures, such as stability, to rank patterns with respect to their ``interestingness'', giving an expert order in which the patterns may be efficiently analyzed.

In this paper,  we develop a novel, rigorous and efficient approach for working with sequential pattern structures in formal concept analysis. The main contributions of this work can be summarized as follows:

\begin{itemize}
	\item
	\textit{Pattern structure specification and analysis.}
	We propose a novel way of dealing with sequences based on complex alphabets by mapping them to pattern structures. The genericity power provided by the pattern structures allows our approach to be directly instantiated with state-of-the-art FCA algorithms, making the final implementation flexible, accurate and scalable. %This approach may be an alternative to the existing methods.
	
	\item 
	\textit{``Projections" for sequential pattern structures}. Projections significantly decrease the number of patterns, while preserving the most interesting ones for an expert. Projections are built to answer questions that an expert may have. Moreover, combinations of projections and concept stability index provide an efficient tool for the analysis of complex sequential datasets. The second advantage of projections is its ability to significantly decrease the complexity of a problem, saving thus computational time.
	
	\item
	\textit{Experimental evaluations.}
	We evaluate our approach on real sequence dataset of a regional healthcare system. The data set contains ordered sets of hospitalizations for cancer patients with information about the hospitals they visited, causes for the hospitalizations and medical procedures. These ordered sets are considered as sequences. The experiments reveal interesting (from a medical point of view) and useful patterns, and show the feasibility and the efficiency of our approach. 
\end{itemize}

This paper is an extension of the work presented at CLA'14 conference~\citep{Buzmakov2013a}. The main differences w.r.t. the CLA'14 paper are a more complete explanation of the mathematical framework and a new experimental part evaluating different aspects of the introduced framework.

The paper is organized as follows. Section~\ref{sect:fca} introduces formal concept analysis and pattern structures. The specification of pattern structures for the case of sequences is presented in Section~\ref{sect:sps}. Section~\ref{sect:projections} describes projections of sequential pattern structures followed in Section~\ref{sect:evaluation} by the evaluation and experimentations. Finally, related works are discussed before concluding the paper.

\section{FCA and pattern structures} \label{sect:fca}
%

% IPS Example
%\newcommand{\Dips}{D_{ips}}
%\newcommand{\sqcapips}{\sqcap_{ips}}
%
%Pattern structures were successfully used for interval data~\citet{Kaytoue2011}. For example, in gene expression data every gene is described by its expression value, i.e. an interval, in a given situation. The meet-semilattice $(\Dips,\sqcapips)$ includes vectors of intervals. An example of an interval pattern structure is given in Table~\ref{tbl:ips-example}. The description of $g_1$, denoted as $g_1^\diamond$, is $g_1^\diamond=\seq{[1,3]\+[3,5]\+[2,4]}$.  The description materializes the fact that the values of parameters $m_1$, $m_2$, $m_3$ are within the corresponding intervals. The similarity operation ($\sqcapips$) between two interval descriptions $g_1^\diamond$ and $g_2^\diamond$ is the component-wise convex hull of intervals. Thus,  $g_1^\diamond \sqcap g_2^\diamond = \seq{[1,7]\+[3, 6]\+[2,5]}$.

%For example, $\set{g_1,g_2}^\diamond=\seq{[1,7]\+[3,6]\+[2,5]}$ and $\seq{[0,7]\+[1,6]\+[1,5]}^\diamond=\set{g_1,g_2}$.
%Based on Table~\ref{tbl:ips-example}, $(\set{g_1\+g_2},\seq{[1,7]\+[3, 6]\+[1,5]})$ and $(\set{g_1},\seq{[1,3]\+[3,5]\+[3,4]})$ are examples of pattern concepts. 

%\input{tex/tbl-ips-example.tex}
%The interval pattern concept lattice resulting from Table~\ref{tbl:ips-example} is shown in Figure~\ref{fig:ips-example}.

	\subsection{Formal concept analysis}\label{sect:fca-fca}
%

%FCA is a formalism for data analysis~\citep{Ganter1999}. 
FCA is a formalism that can be used for guiding data analysis and knowledge discovery~\citep{Ganter1999}.
FCA starts with a formal context and builds a set of formal concepts organized within a concept lattice. A formal context is a triple $(G,M,I)$, where $G$ is a set of objects, $M$ is a set of attributes and $I$ is a relation between $G$ and $M$, $I\subseteq G \times M$.
In Table~\ref{tbl:fca-context}, a cross table for a formal context is shown.
A Galois connection between $G$ and $M$ is defined as follows:
\begin{align*}
	A' &= \{m \in M \mid \forall g \in A, (g,m) \in I\},& A \subseteq G\\
	B' &= \{g \in A \mid \forall m \in M, (g,m) \in I\},& B \subseteq M
\end{align*}
The  Galois connection maps a set of objects to the maximal set of attributes shared by all objects and reciprocally. For example, $\set{g_1\+g_2}'=\{m_4\}$, while $\set{m_4}'=\set{g_1\+g_2\+g_4}$, i.e. the set $\{g_1,g_2\}$ is not maximal. Given a set of objects $A$, we say that $A'$ is the description of $A$.

\begin{table}
	\tbl{A toy FCA context.}{
	\begin{tabular}{r|cccc}
		\toprule
		& $m_1$ & $m_2$ & $m_3$ & $m_4$ \\
		\colrule
		$g_1$
		& x & & & x \\
		$g_2$
		& & & x & x \\
		$g_3$
		& & x & & \\
		$g_4$
		& & & x & x \\
		\botrule
	\end{tabular}}
	\label{tbl:fca-context}
\end{table}
\begin{figure}
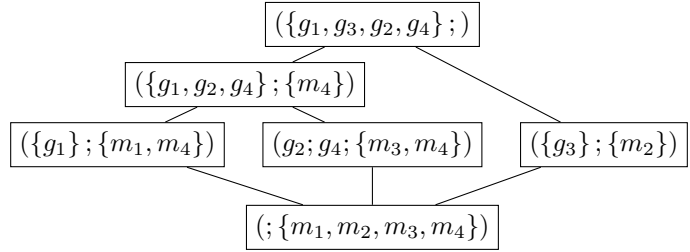

	\centering
	\putFCALattice	
	\caption{Concept Lattice for the toy context}
	\label{fig:fca-lattice}
\end{figure}

\begin{definition}
	A formal concept is a pair $(A,B)$, where $A \subseteq G$ is a subset of objects, $B \subseteq M$ is a subset of attributes, such that $A'=B$ and $A=B'$, where $A$ is called the extent of the concept, and $B$ is called the intent of the concept.
\end{definition}

A formal concept corresponds to a pair of maximal sets of objects and attributes, i.e. it is not possible to add an object or an attribute to the concept without violating the maximality property. For example a pair $(\set{g_1\+g_2\+g_4},\set{m_4})$ is a formal concept.
Formal concepts can be partially ordered w.r.t. the extent inclusion (dually, intent inclusion). For example, 
$(\set{g_1};\set{m_1\+m_4}) \leq (\set{g_1\+g_2\+g_4},\set{m_4}).$ This partial order of concepts is shown in Figure~\ref{fig:fca-lattice}. The number of formal concepts for a given context can be exponential w.r.t. the cardinality of set of objects or set of attributes. It is easy to see that for context $(G,G,I_G)$, where $I_G=\{(x,y) \mid x \in G, y\in G, x \neq y\}$, the number of concepts is equal to $2^{|G|}$.

	\subsection{Stability index of a concept}\label{sect:fca-stability}

The number of concepts in a lattice for real-world tasks can be large. To find the most interesting subset of concepts, different measures can be used such as the stability of the concept~\citep{KuznetsovStability2007} or the concept probability and separation~\citep{Klimushkin2010}. These measures help extracting the most interesting concepts. However, the last ones are less reliable in noisy data. 

\begin{definition}\label{def:stability}
	Given a concept $c$, the concept stability $\stab{c}$ of $c$ is the relative number of subsets of the concept extent (denoted $\extC{c}$), whose description, i.e. the result of $(\cdot)'$, is equal to the concept intent (denoted $\intC{c}$).
	\begin{equation}\label{eq:stability}
		%Reviewer's request
		%\stab{c}:=\frac{|\{s\in\wp(\extC{c})\text{ such that } s' = \intC{c}\}|}{|\wp(\extC{c})|}
		%Sergei's request
		\stab{c}:=\frac{|\{s\in\wp(\extC{c}) \mid s' = \intC{c}\}|}{|\wp(\extC{c})|}
	\end{equation}
\end{definition}

Here $\wp(P)$ is the powerset of $P$. Stability measures how a concept depends on objects in its extent. The larger the stability is the more combinations of objects can be deleted from the context without affecting the intent of the concept, i.e. the intent of the most stable concepts is likely to be a characteristic pattern of a given phenomenon and not an artifact of a dataset. Of course, stable concepts still depend on the dataset, and, consequently some important information can be contained in the unstable concepts. However, the stability can be considered as a good heuristic for selecting concepts because the more stable the concept is the less it depends on the given dataset w.r.t. to object removal.

\begin{example}
	Figure~\ref{fig:stability} shows a lattice for the context in Table~\ref{tbl:stability-context}, for simplicity some intents are not given. Extent of the outlined concept $c$ is $\extC{c}=\set{g_1\+g_2\+g_3\+g_4}$, thus, its powerset contains $2^4$ elements. Descriptions of 5 subsets of $\extC{c}$ ($\set{g_1},\dots,\set{g_4}$ and $\emptyset$) are different from $\intC{c}=\set{m_6}$, while all other subsets of $\extC{c}$ have a common description equal to $\set{m_6}$. So, $\stab{c}=\frac{2^4-5}{2^4}=0.69$.
\end{example}

\begin{table}
	\tbl{A toy formal context}{
	\begin{tabular}{l|cccccc}
		\toprule
		& $m_1$ & $m_2$ & $m_3$ & $m_4$ & $m_5$ & $m_6$ \\
		\colrule
		$g_1$ & x &&&&& x \\
		$g_2$ && x &&&& x \\
		$g_3$ &&& x &&& x \\
		$g_4$ &&&& x && x \\
		$g_5$ &&&&& x &   \\
		\botrule
	\end{tabular}}
	\label{tbl:stability-context}
\end{table}
\begin{figure}
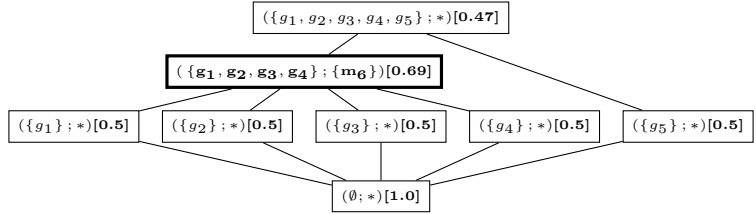

	\centering
	\putStabLattice
	\captionof{figure}{Concept Lattice for the context in Table~\ref{tbl:stability-context} with corresponding stability indexes.}
	\label{fig:stability}
\end{figure}

\newcommand{\DD}{\mathtt{DD}}

One of the fastest algorithm processing a concept lattice $L$ is proposed in~\citep{Roth2008} with the worst-case complexity of $O(|L|^2)$ where $|L|$ is the size of the concept lattice. The experimental section shows that for a big lattice, the stability computation can take much more time than the construction of the concept lattice.
Thus, the estimation of concept stability is an important question. Here we present an efficient way for such an estimation. It should be noticed that in a lattice the extent of any ancestor of a concept $c$ is a superset of the extent of $c$, while the extent of any descendant is a subset. Given a concept $c$ and an immediate descendant $d$, we have $\forall s \subseteq \extC{d}, s'' \subseteq \extC{d}$, which means that $s' \supseteq \intC{d} \supset \intC{c}$, i.e. $s' \neq \intC{c}$. Thus, we can exclude in the computation of the numerator of stability in~(\ref{eq:stability}) all subsets of the extent of a direct descendant $c$. Thus, the following bound holds: \begin{equation}\label{eq:stability-limits}
%	1 - \sum\limits_{ch \in Children}2^{-\diff(c,ch)} \leq
	\stab{c} \leq 1 - \underset{d \in \DD(c)}{\max}\frac{1}{2^{\diff(c,d)}},
\end{equation} where $\DD(c)$ is the set of all direct descendants and $\diff(c,d)$ is the set-difference between extent of  $c$ and extent of $d$, $\diff(c,d)=|\extC{c} \setminus \extC{d}|$. 

\begin{example}\label{ex:stability-diff}
	With help of~(\ref{eq:stability-limits}) we can find all stable concepts (and some unstable), i.e. the concepts with a high stability w.r.t. a threshold $\theta$.
	If $\theta=0.97$, we should compute for each concept $c$ in the lattice the following value $md(c)=\underset{d \in \DD(c)}{\min}\diff(c,d)$ and then select concepts verifying $md(c) \geq -\log(1-0.97)=5.06$.
\end{example}

	\subsection{Pattern structures}\label{sect:fca-pattern-structures}

Although FCA applies to binary contexts, more complex data such as sequences or graphs can be directly processed as well. For that, pattern structures were introduced in~\citet{Ganter2001}.
\begin{definition}
	A pattern structure is a triple $(G,(D,\sqcap),\delta)$, where $G$ is a set of objects, $(D,\sqcap)$ is a complete meet-semilattice of descriptions and $\delta: G \rightarrow D$ maps an object to a description.
\end{definition}

\newcommand{\sqcapb}{\sqcap}

The lattice operation in the semilattice ($\sqcap$) corresponds to the similarity between two descriptions. Standard FCA can be presented in terms of a pattern structure. In this case, $G$ is the set of objects, the semilattice of descriptions is $(\wp(M),\sqcapb)$ and a description is a set of attributes, with the $\sqcapb$ operation corresponding to the set intersection ($\wp(M)$ denotes the powerset of $M$). If $x=\set{a\+b\+c}$ and $y=\set{a\+c\+d}$ then $x \sqcapb y = x \cap y = \set{a\+c}$. The mapping $\delta: G \rightarrow \wp(M)$ is given by, $\delta(g)=\{m \in M \mid (g,m) \in I\}$, and returns the description for a given object as a set of attributes.

The Galois connection for a pattern structure $(G,(D,\sqcap),\delta)$ is defined as follows:
\begin{align*}
	A^\diamond &:= \underset{g \in A}{\bigsqcap}\delta(g), &\text{for } A \subseteq G\\
	d^\diamond &:= \{g \in G \mid d \sqsubseteq \delta(g)\}, &\text{for } d \in D
\end{align*}

The Galois connection makes a correspondence between sets of objects and descriptions. Given a subset of objects $A$, $A^\diamond$ returns the description which is common to all objects in $A$. Given a description $d$, $d^\diamond$ is the set of all objects whose description subsumes $d$.
More precisely, the partial order  (or the subsumption order) on $D$ ($\sqsubseteq$) is defined w.r.t. the similarity operation $\sqcap$:
$c \sqsubseteq d \Leftrightarrow c \sqcap d = c$, and $c$ is subsumed by $d$. 

\begin{definition}
	A pattern concept of a pattern structure $(G,(D,\sqcap),\delta)$ is a pair $(A,d)$ where $A \subseteq G$ and $d \in D$ such that $A^\diamond = d$ and $d^\diamond = A$, $A$ is called the concept extent and $d$ is called the concept intent.
\end{definition}

As in standard FCA, a pattern concept corresponds to the maximal set of objects $A$ whose description subsumes the description $d$, where $d$ is the maximal common description for objects in $A$. 
The set of all concepts can be partially ordered w.r.t. partial order on extents (dually, intent patterns, i.e~$\sqsubseteq$), within a concept lattice.

An example of pattern structures is given in Table~\ref{tbl:sps-ex-patients}, while the corresponding lattice is depicted in Figure~\ref{fig:sps-lattice}.

As stability of concepts only depends on extents, it can be defined by the same procedure for both formal contexts and pattern structures.

\section{Sequential pattern structures}\label{sect:sps}

Certain phenomena, such as a patient trajectory (clinical history), can be considered as a sequence of events. This section describes how FCA and pattern structures can process sequential data.

	\subsection{An example of sequential data}\label{sect:sps-example}

\begin{table}
	\tbl{Toy sequential data on patient medical trajectories.}{
	\begin{tabular}{c|l}
		\toprule
		Patient & Trajectory \\
		\colrule
		$p^1$
		& $\pI$ \\
		$p^2$
		& $\pII$ \\
			$p^3$
		& $\pIII$ \\
		\botrule
	\end{tabular}}
	\label{tbl:sps-ex-patients}
\end{table}
%\begin{minipage}{0.37\columnwidth}
%	\centering
%	\resizebox{\columnwidth}{!}{
%		\putRETaxonomy
%	}
%	\captionof{figure}{Hospital taxonomy for the toy example in Table~\ref{tbl:sps-ex-patients}}
%	\label{fig:sps-ex-taxonomy}
%\end{minipage}

Imagine that we have medical trajectories of patients, i.e. sequences of hospitalizations, where every hospitalization is described by a hospital name and a set of procedures. An example of sequential data on medical trajectories with three patients is given in Table~\ref{tbl:sps-ex-patients}. We have a set of procedures $P=\set{a\+b\+c\+d}$, a set of hospital names $T_H=\{H_1,H_2,H_3,H_4,CL,CH,*\}$, where hospital names are hierarchically organized (by level of generality). $H_1$ and $H_2$ are central hospitals ($CH$), $H_3$ and $H_4$ are clinics ($CL$), and $*$ denotes the root of this hierarchy. The least common ancestor in this hierarchy is denoted by $h_1 \sqcap h_2$, for any $h_1,h_2 \in T_H$, i.e. $H_1 \sqcap H_2 = CH$. Every hospitalization is described by one hospital name and may contain several procedures. The procedure order in each hospitalization is not important in our case. For example, the first hospitalization $\el{H_2}{c\+d}$ for the second patient ($p^2$) was a stay in hospital $H_2$ and during this hospitalization the patient underwent procedures $c$ and $d$.
An important task is to find the ``characteristic'' sequences of procedures and associated hospitals in order to improve hospitalization planning,  optimize clinical processes or detect anomalies.

We approach the search for characteristic sequences by finding the most stable concepts in the lattice corresponding to a sequential pattern structure. For the simplification of calculations, subsequences are considered without ``gaps'', i.e the order of non consequent elements is not taken into account. This is reasonable in this task because experts are interested in regular consecutive events in healthcare trajectories.
A sequential pattern structure is a set of sequences and is based on the set of maximal common subsequences (without gaps) between two sequences. Next subsections define partial order on sequences and the corresponding pattern structures.

	\subsection{Partial order on complex sequences}\label{sect:sps-partial-order}

A sequence is  constituted of elements from an alphabet. The classical subsequence matching task requires no special properties of the alphabet. Several generalizations of the classical case were made by introducing a subsequence relation based on an itemset alphabet~\citep{Agrawal1995} or on a multidimensional and multilevel alphabet~\citep{Plantevit2010}. Here, we generalize the previous cases, requiring for an alphabet to form a semilattice $(E,\sqcape)$
	(We should note that in this paper we consider two semilattices, the first one is related to the characters of the alphabet, $(E,\sqcape)$, and the second one is related to pattern structures, $(D,\sqcap)$).
	Thanks to the formalism of pattern structures we are able to process in a unified way all types of sequential datasets with poset-shaped alphabet (it is mentioned above that any partial order can be transformed into a semilattice). However, some sequential data can have connections between elements, e.g.~\citep{Adda2010}, and, thus, cannot be straightforwardly processed by our approach.

\begin{definition}\label{def:sequence}
	Given a semilattice $(E,\sqcape)$, also called an alphabet, a sequence is an ordered list of elements from $E$. We denote it by $\seq{e_1\+e_2\+\cdots\+e_n}$ where $e_i \in E$.
\end{definition}

In this alphabet semilattice $(E,\sqcape)$ there is a bottom element $\bote$ that can be matched with any other element. Formally, $\forall e \in E, \bote = \bote \sqcape e$. This element is required by the lattice structure, but provides no useful information. Thus, it should be excluded from sequences. The bottom element of $E$ corresponds to the empty set in sequential mining~\citep{Agrawal1995}, and the empty set is always ignored in this domain.

\begin{definition}\label{def:valid-sequence}
	A valid sequence $\seq{e_1\+\cdots\+e_n}$ is a sequence where $e_i \neq \bote$ for all $i \in \{1,\cdots,n\}$ .
\end{definition}

\begin{definition}\label{def:subsequence}
	Given an alphabet $(E,\sqcap_E)$ and two sequences $t=\seq{t_1\+...\+t_k}$ and $s=\seq{s_1\+...\+s_n}$ based on $E$ ($t_q,s_p \in E$), the sequence $t$ is a subsequence of $s$, denoted $t \leq s$, iff $k \leq n$ and there exist $j_1,..j_k$ such that $1 \leq j_1 < j_2 < ... <j_k \leq n$ and for all $i \in \set{1\+2\+...\+k}$, $t_i \sqsubseteqe s_{j_i}$, i.e. $t_i \sqcap_E s_{j_i} =t_i$. 
\end{definition}

\begin{example}
	In the running example (Section~\ref{sect:sps-example}), the alphabet is $E=T_H \times \wp(P)$ with the similarity operation $(h_1,P_1)\sqcap (h_2,P_2) = (h_1 \sqcap h_2, P_1 \cap P_2)$, where $h_1,h_2 \in T_H$ are hospitals and $P_1,P_2 \in \wp(P)$ are sets of procedures. Thus, the sequence $\ssIRef=\ssI$ is a subsequence of $\pIRef=\pI$ because if we set $j_i = i + 1$ (Definition~\ref{def:subsequence}) then 
	$ss^1_1 \sqsubseteq p^1_{j_1}$ (`CH' is more general than $H_1$ and $\set{c\+d} \subseteq \set{c\+d}$), 
	$ss^1_2 \sqsubseteq p^1_{j_2}$ (the same hospital and $\set{b} \subseteq \set{b\+a}$) and 
	$ss^1_3 \sqsubseteq p^1_{j_3}$ (`*' is more general than $H_1$ and $\set{d} \subseteq \set{d}$). 
\end{example}

With complex sequences and this kind of subsequence relation the computation can be hard. Thus, for the sake of simplification, only ``contiguous'' subsequences are considered, where only the order of consequent elements is taken into account, i.e. given $j_1$ in Definition~\ref{def:subsequence}, $j_i = j_{i-1}+1$ for all $i \in \set{2\+3\+...\+k}$. Since experts are interested in regular consecutive events in healthcare trajectories, such a restriction does make sens for our data. It helps to connect only related hospitalizations. 

The next section introduces pattern structures that are based on complex sequences with a general subsequence relation, while the experiments are provided for a ``contiguous'' subsequence relation.

	\subsection{Sequential meet-semilattice}\label{sect:sps-seq-semilattice}

Based on the previous definitions, we can define the sequential pattern structure used for representing and managing sequences. For that, we make an analogy with the pattern structures for graphs~\citep{Kuznetsov1999} where the meet-semilattice operation $\sqcap$ respects subgraph isomorphism. Thus, we introduce a sequential meet-semilattice respecting subsequence relation. Given an alphabet lattice $(E,\sqcape)$, $\mathfrak{S}$ is the set of all valid sequences based on $(E,\sqcape)$. $\mathfrak{S}$ is partially ordered w.r.t. Definition~\ref{def:subsequence}. $(D,\sqcap)$ is a semilattice on $\mathfrak{S}$, where $D \subseteq \wp(\mathfrak{S})$ such that, if $d \in D$ contains a sequence $s$, then all subsequences of $s$ should be included into $d$, $\forall s \in d, \nexists \tilde{s} \leq s: \tilde{s} \notin d$, and the similarity operation is the set intersection for two sets of sequences. Given two patterns $d_1,d_2 \in D$, the set intersection operation ensures that if a sequence $s$ belongs to $d_1 \sqcap d_2$ then any subsequence of $s$ belongs to $d_1 \sqcap d_2$ and thus $d_1 \sqcap d_2 \in D$. As the set intersection operation is idempotent, commutative and associative,  $(D,\sqcap)$ is a semilattice.

\begin{example}
	If pattern $d_1 \in D$ includes sequence $ss^4=\seq{[*,\set{c\+d}]\+[*,\set{b}]}$ (see Table~\ref{tbl:sps-ex-common-ss}), then it should include also $\seq{[*,\set{d}]\+[*,\set{b}]}$, $\seq{[*,\set{c\+d}]}$, $\seq{[*,\set{d}]}$ and others. If pattern $d_2 \in D$ includes $ss^{12}=\seq{[*,\set{a}]\+[*,\set{d}]}$, then it should include $\seq{[*,\set{a}]}$, $\seq{[*,\set{d}]}$ and $\seq{}$. Thus the intersection of two sets $d_1$ and $d_2$ is equal to the set $\set{\seq{[*,\set{d}]},\seq{}}$.
\end{example}

The next proposition stems from the aforementioned and will be used in the proofs in the next section.

\begin{proposition}\label{prop:sequential-subsumption}
	Given $(G,(D,\sqcap),\delta)$ and $x,y \in D$, $x \sqsubseteq y$ if and only if $\forall s^x \in x$ there is a sequence $s^y \in y$, such that $s^x \leq s^y$.
\end{proposition}

The set of all possible subsequences for a given sequence can be large. Thus, it is more efficient to consider a pattern $d \in D$ as a set of only maximal sequences $\tilde{d}$, $\tilde{d}=\{s \in d \mid \nexists s^* \in d: s^* \geq s\}$. Furthermore, every pattern will be given only by the set of all maximal sequences. For example, $\set{p^2} \sqcap \set{p^3} = \set{ss^6\+ss^7\+ss^8}$ (see Tables~\ref{tbl:sps-ex-patients}~and~\ref{tbl:sps-ex-common-ss}), i.e. $\set{ss^6\+ss^7\+ss^8}$ is the set of all maximal sequences specifying the intersection of $p^2$ and $p^3$. Similarly we have $\set{ss^6\+ss^7\+ss^8} \sqcap \set{p^1}= \set{ss^4\+ss^5}$. Note that representing a pattern by the set of all maximal sequences allows for an efficient implementation of the intersection ``$\sqcap$'' of two patterns (in Section~\ref{sect:evaluation-implementation} we give more details on similarity operation w.r.t. a contiguous subsequence relation).

\begin{example}
	The sequential pattern structure for our example (Subsection~\ref{sect:sps-example}) is $(G,(D,\sqcap),\delta)$, where $G=\set{p^1,p^2,p^3}$, $(D,\sqcap)$ is the semilattice of sequential descriptions, and $\delta$ is the mapping associating an object in $G$ to a description in $D$ shown in Table~\ref{tbl:sps-ex-patients}.
	Figure~\ref{fig:sps-lattice} shows the resulting lattice of sequential pattern concepts for this particular pattern structure $(G,(D,\sqcap),\delta)$.
\end{example}

\begin{figure}
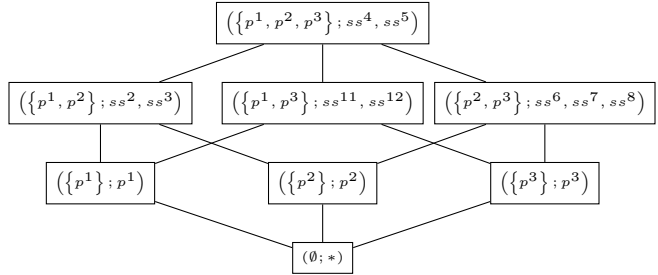

\centering
\putSPSLattice
\caption{
	The concept lattice for the pattern structure given by Table~\ref{tbl:sps-ex-patients}. Concept intents reference to sequences in Tables~\ref{tbl:sps-ex-patients} and~\ref{tbl:sps-ex-common-ss}.
}
\label{fig:sps-lattice}
\end{figure}
\begin{table}
	\tbl{Subsequences of patient sequences in Table~\ref{tbl:sps-ex-patients}.}
	{\begin{tabular}{c|l}\toprule
		& Subsequences \\
		\colrule
		$\ssIRef$
		& $\ssI$ 
		\\
		$\ssIIRef$
		& $\ssII$ 
		\\
		
		$\ssIIIRef$
		& $\ssIII$
		\\
		$\ssIVRef$
		& $\ssIV$
		\\
		
		$\ssVRef$
		& $\ssV$
		\\
		$\ssVIRef$
		& $\ssVI$
		\\
		
		$\ssVIIRef$
		& $\ssVII$
		\\
		$\ssVIIIRef$
		& $\ssVIII$
		\\
		
		$\ssIXRef$
		& $\ssIX$
		\\
		$\ssXRef$
		& $\ssX$
		\\
		
		$\ssXIRef$
		& $\ssXI$
		\\
		$\ssXIIRef$
		& $\ssXII$
		\\
		\botrule	
	\end{tabular}}
	\label{tbl:sps-ex-common-ss}
\end{table}

\section{Projections of sequential pattern structures}\label{sect:projections}

Pattern structures are hard to process due to the large number of concepts in the concept lattice, the complexity of the involved descriptions and the similarity operation. Moreover, a given pattern structure can produce a lattice with a lot of patterns which are not interesting for an expert. \textit{Can we save computational time by avoiding to compute ``useless'' patterns?} Projections of pattern structures ``simplify'' to some degree the computation and allow one to work with a reduced description. In fact, projections can be considered as filters on patterns respecting mathematical properties. These properties ensure that the projection of a semilattice is a semilattice and that projected concepts are related to original ones~\citep{Ganter2001}. Moreover, the stability measure of projected concepts never decreases w.r.t the original concepts.
We introduce projections on sequential patterns revising~\citet{Ganter2001}.
It is necessary to provide an extended definition of projection in order to deal with interesting projections for real-world sequential datasets.

\begin{definition}[\citet{Ganter2001}]\label{def:projection}
	A projection $\psi: D \rightarrow D$ is an interior operator, i.e. it is (1)~monotone ($x \sqsubseteq y \Rightarrow \psi(x) \sqsubseteq \psi(y)$), (2)~contractive ($\psi(x) \sqsubseteq x$) and (3)~idempotent ($\psi(\psi(x))=\psi(x)$).
\end{definition}

\begin{definition}\label{def:projected-PS}
 A projected pattern structure $\psi((G,(D,\sqcap),\delta))$ is a pattern structure $(G,(D_\psi,\sqcap_\psi),\psi \circ \delta)$, where $D_\psi=\psi(D)=\{d\in D \mid \exists d^* \in D: \psi(d^*) = d\}$ and $\forall x,y \in D, x \sqcap_\psi y := \psi( x \sqcap y )$. 
\end{definition}

Note that in~\citep{Ganter2001} $\psi((G,(D,\sqcap),\delta))=(G,(D,\sqcap),\psi \circ \delta)$. Our definition allows one to use a wider set of projections. In fact all projections that we describe for sequential pattern structures below require Definition~\ref{def:projected-PS}. Now we should show that $(D_\psi, \sqcapb_\psi)$ is a semilattice.

\newcommand{\sqcappsi}{\sqcap_\psi}

\begin{proposition}\label{prop:psi(xy)=psi(psi(x)y)}
	Given a semilattice $(D,\sqcap)$ and a projection $\psi$, for all $x,y \in D$ $\psi(x \sqcap y) = \psi( \psi(x) \sqcap y )$.
\end{proposition}
\begin{proof}
	\begin{enumerate}
		\item
		$\psi(x) \sqsubseteq x$, thus, $x,y \sqsupseteq (x \sqcap y) \sqsupseteq (\psi(x) \sqcap y) \sqsupseteq \psi(\psi(x) \sqcap y )$
		\item 
		$x \sqsubseteq y \Rightarrow \psi(x) \sqsubseteq \psi(y)$, thus, $\psi(x \sqcap y) \sqsupseteq \psi(\psi(x) \sqcap y)$
		\item 
		$\psi(x \sqcap y) \sqcap \psi(x) \sqcap y
		\underset{\psi(x \sqcap y) \sqsubseteq \psi(x)}{=}
		\psi(x \sqcap y) \sqcap y
		\underset{\psi(x \sqcap y) \sqsubseteq y}{=}
		\psi(x \sqcap y)$,\\
		then $(\psi(x) \sqcap y) \sqsupseteq \psi(x \sqcap y)$ and $\psi(\psi(x) \sqcap y) \sqsupseteq \psi(\psi(x \sqcap y)) = \psi(x \sqcap y)$
		\item 
		From (2) and (3) it follows that $\psi(x \sqcap y) = \psi( \psi(x) \sqcap y )$.
	\end{enumerate}
\end{proof}

\begin{corollary}\label{cor:psi-sqcap=sqcap-psi}
	$X_1 \sqcappsi X_2 \sqcappsi \cdots \sqcappsi X_N = \psi(X_1 \sqcap X_2 \sqcap \cdots \sqcap X_N)$
\end{corollary}
\begin{proof}
	It can be prooven by induction.	
	\begin{enumerate}
		\item
			$X_1 \sqcappsi X_2 = \psi(X_1 \sqcap X_2)$ by Definition~\ref{def:projected-PS}.
		\item
			If $X_1 \sqcappsi \cdots \sqcappsi X_K = \psi(X_1 \sqcap \cdots \sqcap X_K)$, then

			$\begin{aligned}
				X_1 \sqcappsi \cdots &\sqcappsi X_{K} \sqcappsi X_{K+1} = \psi(X_1 \sqcap \cdots \sqcap X_K) \sqcappsi X_{K+1} = \\
				&=\psi(\psi(X_1 \sqcap \cdots \sqcap X_K) \sqcap X_{K+1}) \underset{\text{Proposition~\ref{prop:psi(xy)=psi(psi(x)y)}}}{=} \psi(X_1 \sqcap \cdots \sqcap X_{K+1})
			\end{aligned}$
	\end{enumerate}
\end{proof}
\begin{corollary}
	Given a semilattice $(D,\sqcap)$ and a projection $\psi$, $(D_\psi, \sqcap_\psi)$ is a semilattice, i.e. $\sqcap_\psi$ is commutative, associative and idempotent.
\end{corollary}

The concepts of a pattern structure and a projected pattern structure are connected through Proposition~\ref{prop:proj-and-concepts}. This proposition can be found in \citet{Ganter2001}, but thanks to Corollary~\ref{cor:psi-sqcap=sqcap-psi}, it is valid in our case.
\begin{proposition}\label{prop:proj-and-concepts}
	Given a concept $(A,d)$ in $\psi((G,(D,\sqcap),\delta))$, the extent $A$ is an extent in $(G,(D,\sqcap),\delta)$. 
	Given a concept $(A,d_\psi)$ in $\psi((G,(D,\sqcap),\delta))$, the intent $d_\psi$ is of the form $d_\psi=\psi(d)$, where $(A,d)$ is a concept in $(G,(D,\sqcap),\delta)$.
\end{proposition}

Moreover, while preserving the extents of some concepts, projections cannot decrease the stability of the projected concepts, i.e. if the projection preserves a stable concept, then its stability (Definition~\ref{def:stability}) can only increase.

\begin{proposition}\label{prop:proj-and-stability}
	Given a pattern structure $(G,(D,\sqcap),\delta)$, its concept $c$ and a projected pattern structure $(G,(D_\psi,\sqcappsi),\psi \circ \delta)$, and the projected concept $\tilde{c}$, if the concept extents are equal ($\extC{c}=\extC{\tilde{c}}$) then $\stab{c} \leq \stab{\tilde{c}}$.
\end{proposition}
\begin{proof}
	Concepts $c$ and $\tilde{c}$ have the same extent. Thus, according to Definition~\ref{def:stability}, in order to prove the proposition, it is enough to prove that for any subset $A \subseteq \extC{c}$, if $A^\diamond = \intC{c}$ in the original pattern structure, then $A^\diamond = \intC{\tilde{c}}$ in the projected one.
	
	Suppose that $\exists A \subset \extC{c}$ such that $A^\diamond = \intC{c}$ in the original pattern structure and $A^\diamond \neq \intC{\tilde{c}}$ in the projected one. Then there is a descendant concept $\tilde{d}$ of $\tilde{c}$ in the projected pattern structure such that $A^\diamond = \intC{\tilde{d}}$ in the projected lattice. Then there is an original concept $d$ for the projected concept $\tilde{d}$ with the same extent $\extC{d}$. Then $A^\diamond \sqsupseteq \intC{d}\sqsupset \intC{c}$ and, so, $A^\diamond$ cannot be equal to $\intC{c}$ in the original lattice. Contradiction.
\end{proof}

Now we are going to present two projections of sequential pattern structures. The first projection comes from the following observation. In many cases it may be more interesting to analyze quite long subsequences rather than short ones. This kind of projections is called \textit{Minimal Length Projection} (MLP) and it depends on the minimal length parameter $\ell$ for the sequences in a pattern. The corresponding function $\psi$ maps a pattern without short sequences to itself, and a sequence with short sequences to the pattern containing only long sequences w.r.t. a given length threshold. Later, propositions~\ref{prop:sequential-subsumption}~and~\ref{prop:proj-MLP} state that MLP is coherent with Definition~\ref{def:projection}.

\begin{definition}
	 The function $\psi_{MLP}:D \rightarrow D$ of minimal length $\ell$ is defined as \[
		\psi_{MLP}(d) = \set{s \in d \mid length(s) \geq \ell}
	\]
\end{definition}

\begin{example}
	If we prefer common subsequences of length $\ell \geq 3$, then between $p^2$ and $p^3$ in Table~\ref{tbl:sps-ex-patients} there is only one maximal common subsequence, $ss^6$ in Table~\ref{tbl:sps-ex-common-ss}, while $ss^7$ and $ss^8$ are too short to be considered.
	Figure~\ref{fig:sps-ex-proj-lattices:mlp} shows the lattice of the projected pattern structure (Table~\ref{tbl:sps-ex-patients}) with patterns of length greater or equal to $3$.
\end{example}

\begin{proposition}\label{prop:proj-MLP}
	The function $\psi_{MLP}$ is a monotone, contractive and idempotent function on the semilattice $(D,\sqcap)$.
\end{proposition}
\begin{proof}
	The contractivity and idempotency are quite clear from the definition. It remains to prove the monotonicity.
	
	If $X \sqsubseteq Y$, where $X$ and $Y$ are sets of sequences, then for every sequence $x \in X$ there is a sequence $y \in Y$ such that $x \leq y$ (Proposition~\ref{prop:sequential-subsumption}). We should show that $\psi(X) \sqsubseteq \psi(Y)$, or in other words for every sequence $x \in \psi(X)$ there is a sequence $y \in \psi(Y)$, such that $x\leq y$. Given $x \in \psi(X)$, since $\psi(X)$ is a subset of $X$ and $X \sqsubseteq Y$, there is a sequence $y\in Y$ such that $x \leq y$, with $|y| \geq |x| \geq \ell$ ($\ell$ is a parameter of MLP), and thus, $y \in \psi(Y)$.
\end{proof}

Another important type of projections is related to a variation of the lattice alphabet $(E,\sqcape)$. One possible variation of the alphabet is to ignore certain fields in the elements. For example, if a hospitalization is described by a hospital name and a set of procedures, then either hospital or procedures can be ignored in similarity computation. For that, in any element the set of procedures should be substituted by $\emptyset$, or the hospital by $*$ (``arbitrary hospital'') which is the most general element of the taxonomy of hospitals.

Another variation of the alphabet is to require that some field(s) should not be empty. For example, we want to find patterns with non-empty set of procedures or the element $*$ of the hospital taxonomy is not allowed in elements of a sequence. Such variations are easy to realize within our approach. For this, when computing the similarity operation between elements of the alphabet, one should check if the result contains empty fields and, if yes, should substitute the result by $\bot$. This variation is useful, as it is shown in the experimental section, but is rather difficult to define within more classical frequent sequence mining approaches, which will be discussed later.

\begin{example}\label{ex:alphabet-proj-no-procedures}
	An expert is interested in finding sequential patterns describing how a patient changes hospitals, but with little interest in procedures. Thus, any element of the alphabet lattice, containing a hospital and a non-empty set of procedures can be projected to an element with the same hospital, but with an empty set of procedures. 
\end{example}

\begin{example}\label{ex:alphabet-proj-no-top-hospitals}
	An expert is interested in finding sequential patterns containing some information about the hospital in every hospitalization, and the corresponding procedures, i.e. hospital field in the patterns cannot be equal to $*$, e.g., $ss^5$ is an invalid pattern, while $ss^6$ is a valid pattern in Table~\ref{tbl:sps-ex-common-ss}. Thus, any element of the alphabet semilattice with $*$ in the hospital field can be projected to the $\bote$.
	Figure~\ref{fig:sps-ex-proj-lattices:NoRootHospital} shows the lattice corresponding to the projected pattern structure (Table~\ref{tbl:sps-ex-patients}) defined by a projection of the alphabet semilattice.
\end{example}

Below we formally define how the alphabet projection of a sequential pattern structure should be processed. Intuitively, every sequence in a pattern should be substituted with another sequence, by applying the alphabet projection to all its elements. However, the result can be an incorrect sequence, because $\bote$ cannot belong to a valid sequence. Thus, sequences in a pattern should be ``developed'' w.r.t. $\bote$, as it is explained below.

%\begin{definition}\label{def:proj-subsequence}
%	Given an alphabet $(E,\sqcape)$, an alphabet projection $\psi$, a sequence $t=\seq{t_1\+...\+t_k}$ is a subsequence of a sequence $s=\seq{s_1\+...\+s_n}$ under a projection $\psi$, denoted $t \leq_\psi s$, iff $k \leq n$ and there exists $j_1,..j_k$ such that $1 \leq j_1 < j_2 < ... <j_k \leq n$ and for all $i \in \set{1\+2\+...\+k}$, $t_i \sqsubseteqe \psi(s_{j_i})$. 
%\end{definition}
%
%This definition is quite similar to Definition~\ref{def:subsequence}, however the larger sequence is considered under the alphabet projection. But having thi

\begin{definition}\label{def:projected-sequence}
	Given an alphabet $(E,\sqcape)$, a projection of the alphabet $\psi$ and a sequence $s=\seq{s_1,\cdots,s_n}$ based on $E$, the projection $\psi(s)$ is the sequence $\tilde{s}=\seq{\tilde{s}_1,\cdots,\tilde{s}_n}$, such that $\tilde{s}_i=\psi(s_i)$.
\end{definition}

Here, it should be noticed that $\tilde{s}$ is not necessarily a valid sequence (see Definition~\ref{def:valid-sequence}), since it can include $\bote$ as an element. However, in sequential pattern structures, elements should include only valid sequences (see Section~\ref{sect:sps-seq-semilattice}).

\begin{definition}\label{def:pattern-alph-proj}
	Given an alphabet $(E,\sqcape)$, a projection of the alphabet $\psi_E$, an alphabet projection for the sequential pattern structure $\psi(d)$ is the set of valid sequences smaller than the projected sequences from $d$:\[
		\psi(d) = \{s \in \mathfrak{S} | (\exists t \in d) s \leq \psi_E(t)\},
	\] where $\mathfrak{S}$ is the set of all valid sequences based on $(E,\sqcape)$.
\end{definition}

\begin{example}
	$\{\ssVIRef\}=\set{\ssVI}$ is an al\-pha\-bet-pro\-jec\-ted pattern for the pattern $\{\ssXRef\}=\set{\ssX}$, where the alphabet lattice projection is given in Example~\ref{ex:alphabet-proj-no-top-hospitals}.
	
	In the case of contiguous subsequences, $\set{\seq{\el{CH}{c\+d}}}$ is an al\-pha\-bet-pro\-jec\-ted pattern for the pattern $\{\ssIIRef\}=\set{\ssII}$, where the alphabet lattice projection is given by projecting every element with medical procedure $b$ to the element with the same hospital and with the same set of procedures excluding $b$. The projection of sequence $\ssIIRef$ is $\seq{\el{CH}{c\+d}\+\el{*}{}\+\el{*}{d}}$, but $\el{*}{} = \bote$, and, thus, in order to project the pattern $\{\ssIIRef\}$ the projected sequence is substituted by its maximal subsequences, i.e. \[
		\psi(\{\ssII\})=\set{\seq{\el{CH}{c\+d}}}.
	\]
\end{example}

\begin{proposition}\label{prop:proj-alphabet-proj}
	Considering an alphabet $(E,\sqcape)$, a projection of the alphabet $\psi$, a sequential pattern structure $(G,(D,\sqcap),\delta)$, the alphabet projection (see Definition~\ref{def:pattern-alph-proj}) is monotone, contractive and idempotent.
\end{proposition}
\begin{proof}
	This projection is idempotent, since the projection of the alphabet is idempotent and only the projection of the alphabet can change the elements appearing in sequences.
	
	It is contractive because for any pattern $d \in D$ and any sequences $s \in d$, a projection of the sequence $\tilde{s}=\psi(s)$ is a subsequence of $s$. In Definition~\ref{def:pattern-alph-proj} the projected sequences should be substituted by their subsequences in order to avoid $\bote$, building the sets $\{\tilde{s}^i\}$. Thus, $s$ is a supersequence for any $\tilde{s}^i$, and, so, the projected pattern $\tilde{d}=\psi(d)$ is subsumed by the pattern $d$.
	
	Finally, we should show monotonicity. Given two patterns $x,y \in D$, such that $x \sqsubseteq y$, i.e. $\forall s^x \in x, \exists s^y \in y: s^x \leq s^y$, consider the projected sequence of $s^x$, $\psi(s^x)$. As $s^x \leq s^y$ for some $s^y$ then for some $j_0 < \cdots < j_{|s^x|}$ (see Definition~\ref{def:subsequence}) $s^x_i \sqsubseteqe s^y_{j_i}$ ($i \in {1,2,...,|s^x|}$), then $\psi(s^x_i) \sqsubseteqe \psi(s^y_{j_i})$ (by the monotonicity of the alphabet projection), i.e. the projected sequence preserves the subsequence relation. Thus, the set of allowed subsequences of $s^x$ is a subset of the set of allowed subsequences of $s^y$. Hence, the alphabet projection of the pattern preserves pattern subsumption relation, $\psi(x) \leq \psi(y)$ (Proposition~\ref{prop:sequential-subsumption}), i.e. the alphabet projection is monotone.
\end{proof}

\begin{figure}
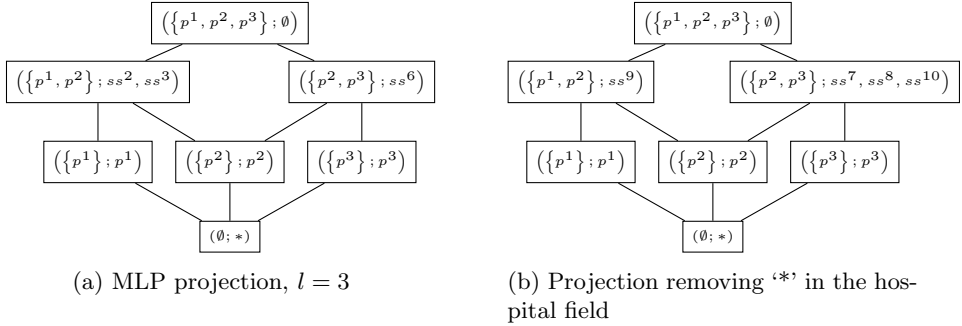

	\centering
	\begin{subfigure}[t]{0.45\columnwidth}
		\centering
		\putSPSLatticeMLPIII
		\caption{MLP projection, $l=3$}
		\label{fig:sps-ex-proj-lattices:mlp}
	\end{subfigure}
	\hfill
	\begin{subfigure}[t]{0.45\columnwidth}
		\centering
		\putSPSLatticeNoTypeRoot
		\caption{Projection removing `*' in the hospital field}
		\label{fig:sps-ex-proj-lattices:NoRootHospital}
	\end{subfigure}
	\caption{The projected concept lattices for the pattern structure given by Table~\ref{tbl:sps-ex-patients}. Concept intents refer to the sequences in Tables~\ref{tbl:sps-ex-patients}~and~\ref{tbl:sps-ex-common-ss}.}
	\label{fig:sps-ex-proj-lattices}
\end{figure}

\section{Sequential pattern structure evaluation}\label{sect:evaluation}

	\subsection{Implementation}\label{sect:evaluation-implementation}

Nearly any state-of-the-art FCA algorithm can be adapted to process pattern structures. We adapted the \texttt{AddIntent} algorithm~\citep{Merwe2004}, as the lattice structure is important for us to calculate stability (see an algorithm for calculating stability in~\citep{Roth2008}). To adapt the algorithm to our needs, every set intersection operation on attributes is substituted with the semilattice operation $\sqcap$ on corresponding patterns, while every subset checking operation is substituted with the semilattice order checking $\sqsubseteq$, in particular all $(\cdot)'$ are substituted with $(\cdot)^\diamond$.

The next question is how the semilattice operation $\sqcap$ and subsumption relation $\sqsubseteq$ can be implemented for contiguous sequences.
Given two sets of sequences $S=\{s^1,...s^n\}$ and $T=\{t^1,...,t^m\}$, the similarity of these sets $S \sqcap T$, is calculated according to Section~\ref{sect:sps-seq-semilattice}, i.e. maximal sequences among all common subsequences for any pair of sequences $s^i$ and $t^j$.

To find all common subsequences of two sequences, the following observations can be useful. If $ss=\seq{ss_1 \+ ... \+ ss_l}$ is a subsequence of $s=\seq{s_1\+...\+s_n}$ with $j^s_i=k^s + i$, i.e. $ss_i \sqsubseteq_E s_{k^s+i}$ (Definition~\ref{def:subsequence}: $k^s$ is the index difference from which $ss$ is a contiguous subsequence of $s$) and a subsequence of $t=\seq{t_1\+...\+t_m}$ with $j^t_i=k^t+i$, i.e. $ss_i \sqsubseteq_E t_{k^t+i}$, then for any index $i\in \set{1\+2\+...\+l}$, $ss_i \sqsubseteq_E (s_{j^s_i} \sqcap t_{j^t_i})$. Thus, to find all maximal common subsequences of $s$ and $t$, we first align $s$ and $t$ in all possible ways. For each alignment of $s$ and $t$ we compute the resulting intersection. Finally, we keep only the maximal intersected subsequences.

For example, let us consider two possible alignments of $s^1$ and $s^2$:\\
\begin{minipage}{0.55\columnwidth}
	\begin{tabular}{rr@{}c@{}c@{}c@{}l}
		$s^1=$
		& $\seqbegin\set{a}\seqdlmt\set{c\+d}\seqdlmt\f$ 
		& $\set{b\+a}$ & $\seqdlmt$ & $\set{d}$ 
		& $\s\seqend$\\
		
		$s^2=$
		& $\seqbegin\f$
		& $\set{c\+d}$ & $\seqdlmt$ & $\set{b\+d}$
		& $\s\seqdlmt\set{a\+d}\seqend$\\
		
		$ss^l=$
		& $\seqbegin\f$
		& $\emptyset$ & $\seqdlmt$ & $\set{d}$
		& $\s\seqend$
	\end{tabular}
\end{minipage}
\begin{minipage}{0.42\columnwidth}
	\begin{tabular}{rr@{}c@{$\seqdlmt$}c@{$\seqdlmt$}c@{}l}
		$s^1=$
		& $\seqbegin\set{a}\seqdlmt\f$ 
		& $\set{c\+d}$ & $\set{b\+a}$ & $\set{d}$ 
		& $\s\seqend$\\
		
		$s^2=$
		& $\seqbegin\f$
		& $\set{c\+d}$ & $\set{b\+d}$ & $\set{a\+d}$
		& $\s\seqend$\\
		
		$ss^r=$
		& $\seqbegin\f$
		& $\set{c\+d}$ & $\set{b}$ & $\set{d}$
		& $\s\seqend$
	\end{tabular}
\end{minipage}
\\
The left intersection $ss^l$ is not retained, as it is not maximal ($ss^l < ss^r$), while the right intersection $ss^r$ is kept.

The complexity of the alignment for two sequences $s$ and $t$ is $O(|s|\cdot |t| \cdot \gamma)$, where $\gamma$ is the complexity of computing a common ancestor in the alphabet lattice $(E,\sqcap)$.

	\subsection{Experiments and discussion}\label{sect:evaluation-experiments}

The experiments are carried out on a MacBook Pro with a 2.5GHz Intel Core i5, 8GB of RAM Memory running OS X 10.6.8. The algorithms are not parallelized and are coded in C++.

%First, the public available database from UCI repository on anonymous web data is used as a benchmark dataset for scalability tests. This database contains around $10^6$ transactions, and each transaction is a sequence based on ``simple'' alphabet, i.e. with no order on the elements. The overall time changes from $37279$ seconds for the sequences of length $MLP \geq 5$ upto $97042$ seconds for the sequences of length $MLP \geq 3$. 
%For more details see the web-page.\footnote{\url{http://www.loria.fr/~abuzmako/ICDM2013/experiment-uci.html}}

%% NICO
%% modification of the examples so they are more consistent (chu nancy, ch metz, etc...)

Our use-case dataset comes from a French healthcare system, called PMSI\footnote{Programme de M\'edicalisation des Syt\`emes d'Information.} \citep{Fetter1980}. Each element of a sequence has a ``complex'' nature. The dataset contains $500$ patients suffering from \emph{lung cancer}, who live in the Lorraine region (Eastern France). Every patient is described as a sequence of hospitalizations without any time-stamp. A hospitalization is a tuple with three elements: (i) healthcare institution (e.g. university hospital of Nancy ($CHU_{Nancy}$)), (ii) reason for the hospitalization (e.g. a cancer disease), and (iii) set of medical procedures that the patient undergoes. An example of a medical trajectory is given below:

{\small\[\begin{aligned}
	\seq{\el{\text{CHU}_{Nancy},\text{Cancer}}{mp_1,mp_2}\f&\s\+\f
	\s\el{\text{CH}_{Paris},\text{Chemo}}{}\f&\s\+\el{\text{CH}_{Paris},\text{Chemo}}{}}.
\end{aligned}\]}

This sequence represents a patient trajectory with three hospitalizations. It expresses that the patient was first admitted to the university hospital of Nancy ($CHU_{Nancy}$) for a cancer problem as a reason, and underwent procedures $mp_1$ and $mp_2$. Then he had two consequent hospitalizations in the general hospital of Paris ($CH_{Paris}$) for chemotherapy with no additional procedure. Substituting the same consequent hospitalizations by the number of repetitions, we have a shorter and more understandable trajectory. For example, the above pattern is transformed into two hospitalizations where the first hospitalization repeats once and the second twice:

{\small\[
	\seq{\el{\text{CHU}_{Nancy},\text{Cancer}}{mp_1,mp_2} \times [1]\+\el{\text{CH}_{Paris},\text{Chemo}}{} \times [2]}.
\]}

Diagnoses are coded according to the 10$^{th}$ International Classification of Diseases (ICD10). Based on this coding, diagnoses could be described at 5 levels of granularity: root, chapter, block, 3-character, 4-character, terminal nodes. This taxonomy has $1544$ nodes. The healthcare institution is associated with a geographical taxonomy of 4 levels, where the first level refers to the root (France) and the second, the third and the fourth levels correspond to administrative region, administrative department and hospital respectively. Figure~\ref{fig:FinessTax} presents University Hospital of Nancy (code: 540002078) as a hospital in Meurthe et Moselle, which is a department in Lorraine, region of France. This taxonomy has $304$ nodes. The \textit{medical procedures} are coded according to the French nomenclature ``Classification Commune des Actes M\'edicaux (CCAM)''. The distribution of sequence lengths is shown in Figure~\ref{fig:length-distrib}. 

\begin{figure}
	\centering
	\includegraphics[width=\columnwidth]{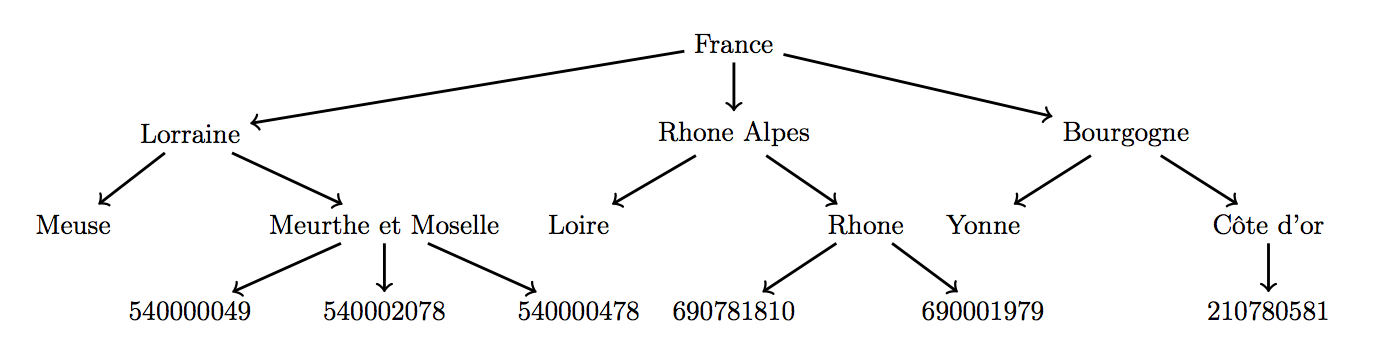}
	\caption{A geographical taxonomy of the healthcare institution}
	\label{fig:FinessTax}
\end{figure}

\begin{figure}
	\centering
	\includegraphics[width=\columnwidth]{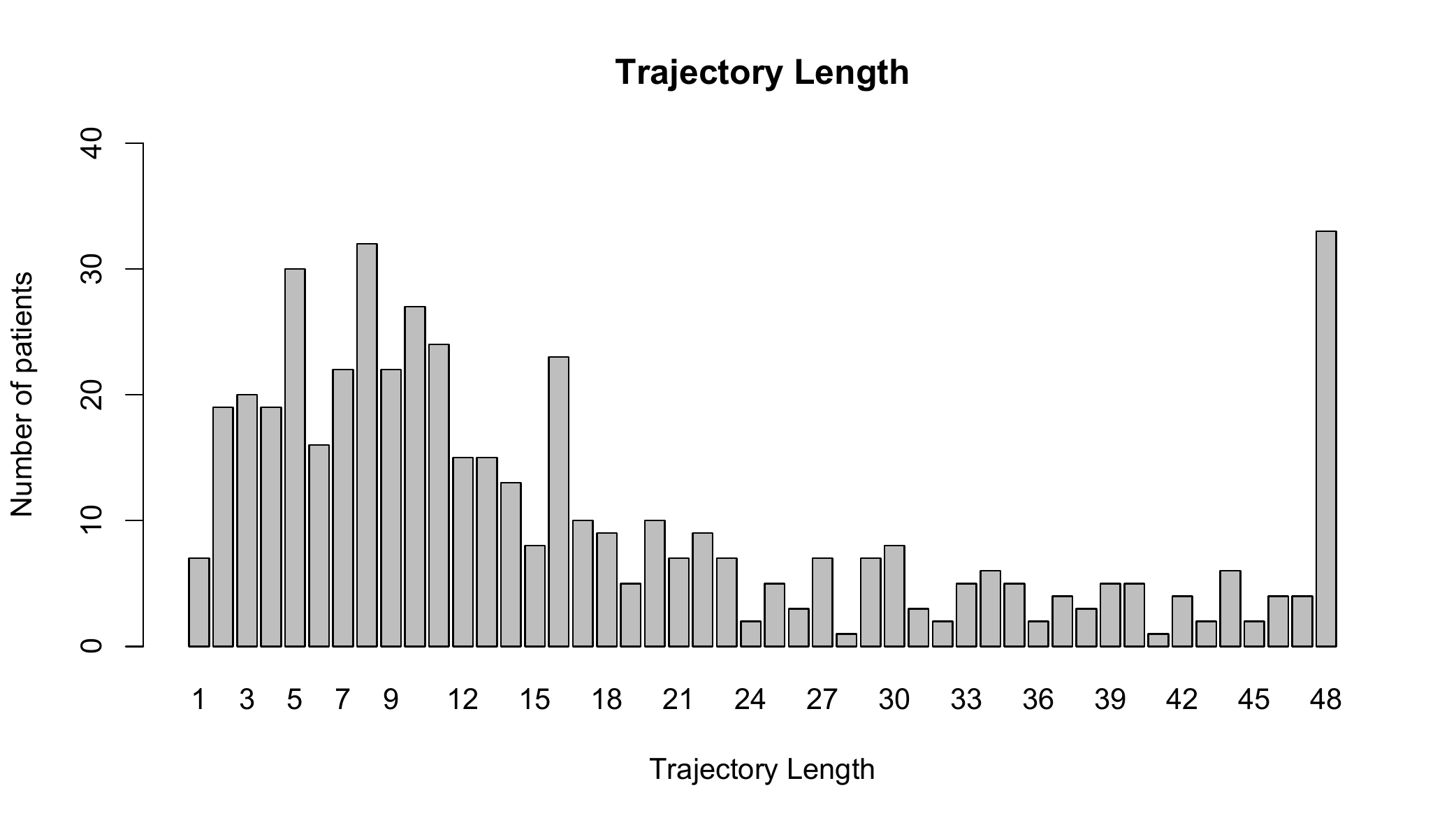}
	\caption{The length distribution of sequences in the dataset}
	\label{fig:length-distrib}
\end{figure}

%For this dataset the computation of the whole lattice is infeasible. However a medical expert is not interested in all possible patterns, but rather in patterns answering his/her analysis question(s). First of all, an expert may know the minimal size of sequences he/she is interested in, i.e. the threshold for MLP projection can be set. If an expert is interested in sequential patterns, the patterns of length $1$ are unlikely to be of interest (knowing that people go to hospitals when they are sick is not a valuable new knowledge). Thus, we use the MLP projection of length $2$ and $3$, taking into account the small average length of the sequences (see Figure~\ref{fig:length-distrib}).

With 500 patient trajectories, the computation of the whole lattice is infeasible. We are not interested in all possible frequent trajectories, but rather in trajectories which answer medical analysis questions. An expert may know the minimal size of trajectories that he is interested in, i.e. setting the MLP projection. We use the MLP projection of length $2$ and $3$ and take into account that most of the patients has at least 2 hospitalizations in the trajectory (see Figure~\ref{fig:length-distrib}).

%
%\subsubsection{Pattern Structure Approach}
%

\tikznamedpicture{\ProjIITime}{
	\node [inner xsep=0, inner ysep=0] (n) {\includegraphics[width=\columnwidth]{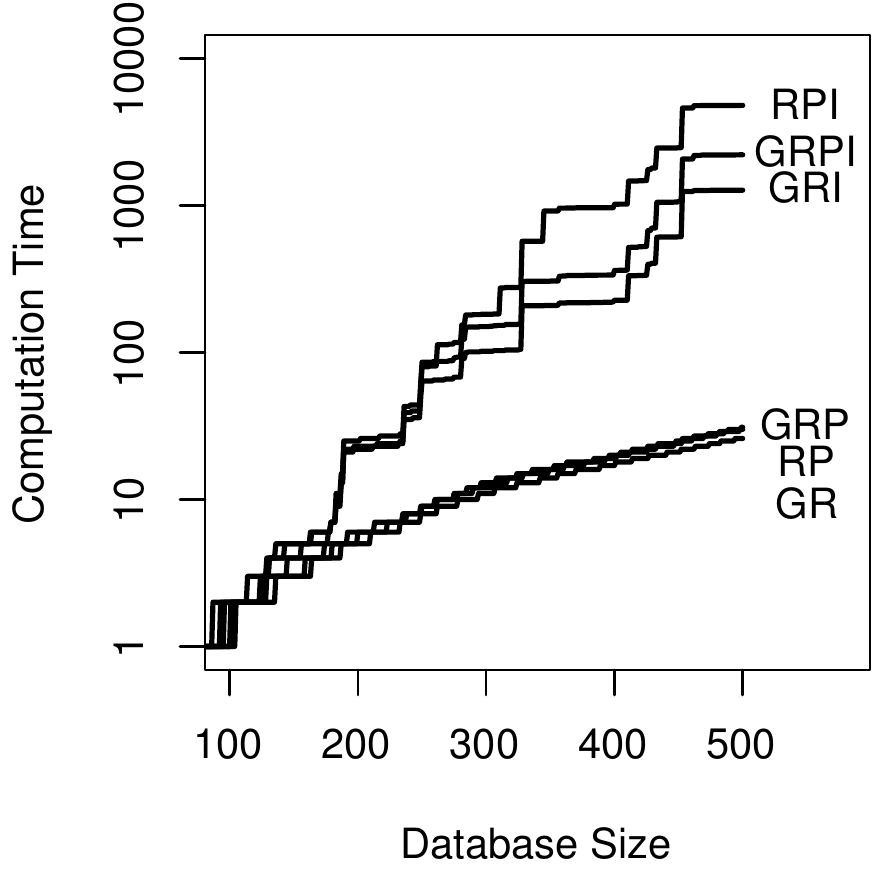}};
	\node [xshift=-2.95cm, yshift = 2.2cm, rotate=90] {(s)};
}
\tikznamedpicture{\ProjIIITime}{
	\node [inner xsep=0, inner ysep=0] (n) {\includegraphics[width=\columnwidth]{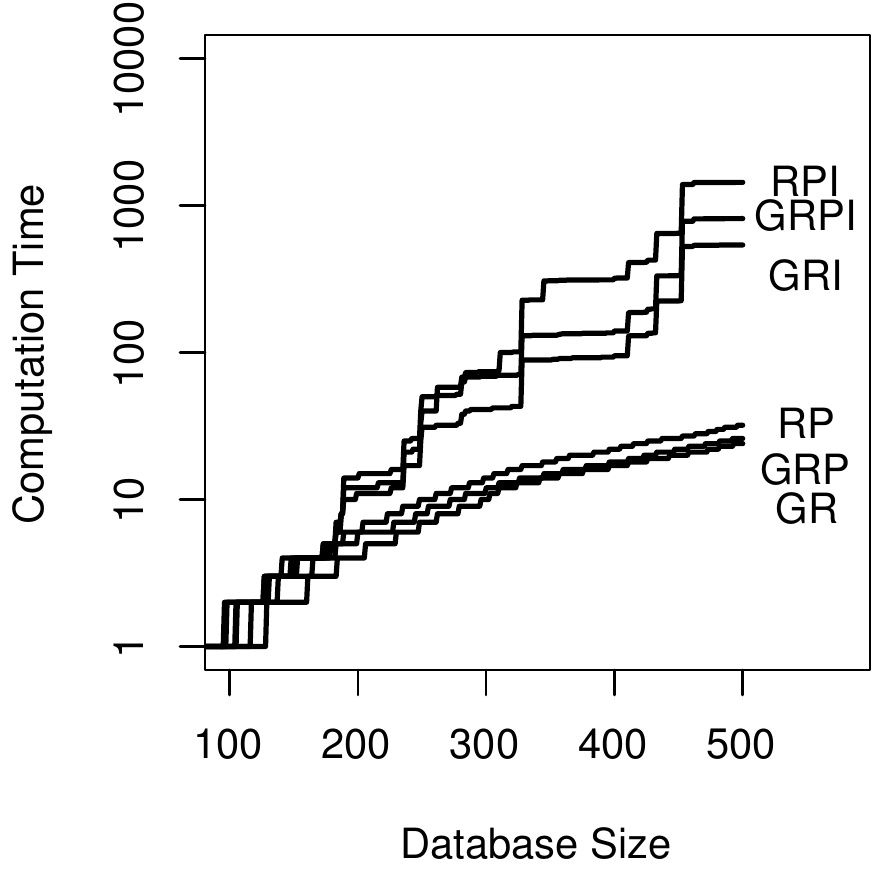}};
	\node [xshift=-2.95cm, yshift = 2.2cm, rotate=90] {(s)};
}
\begin{figure}
	\centering
	\begin{subfigure}[b]{0.45\columnwidth}
		\ProjIITime
		\caption{MLP projection, $\ell=2$}
		\label{fig:computational-stats-ell2}
	\end{subfigure}
	\qquad
	\begin{subfigure}[b]{0.45\columnwidth}
		\ProjIIITime
		\caption{MLP projection, $\ell=3$}
		\label{fig:computational-stats-ell3}
	\end{subfigure}
	\caption{Computational time for different projections}
	\label{fig:computational-stats}
\end{figure}

\begin{figure}
	\centering
	\begin{subfigure}[b]{0.45\columnwidth}
		\includegraphics[width=\columnwidth]{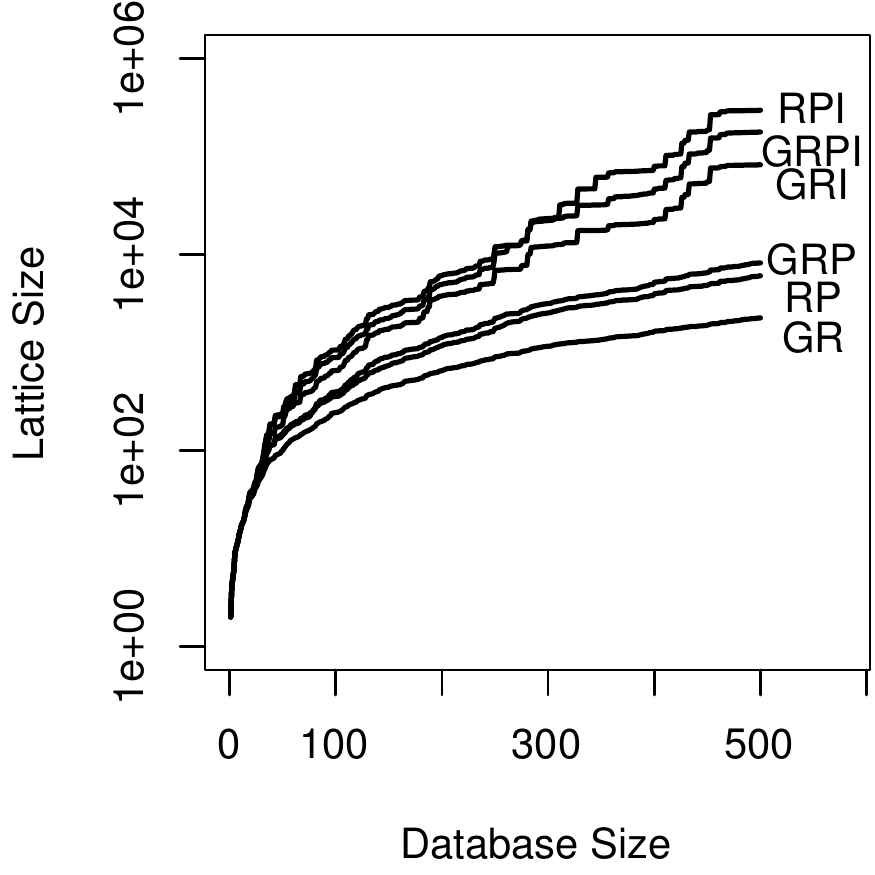}	
		\caption{MLP projection, $\ell=2$}
		\label{fig:lattice-size-stats-ell2}
	\end{subfigure}
	\begin{subfigure}[b]{0.45\columnwidth}
		\includegraphics[width=\columnwidth]{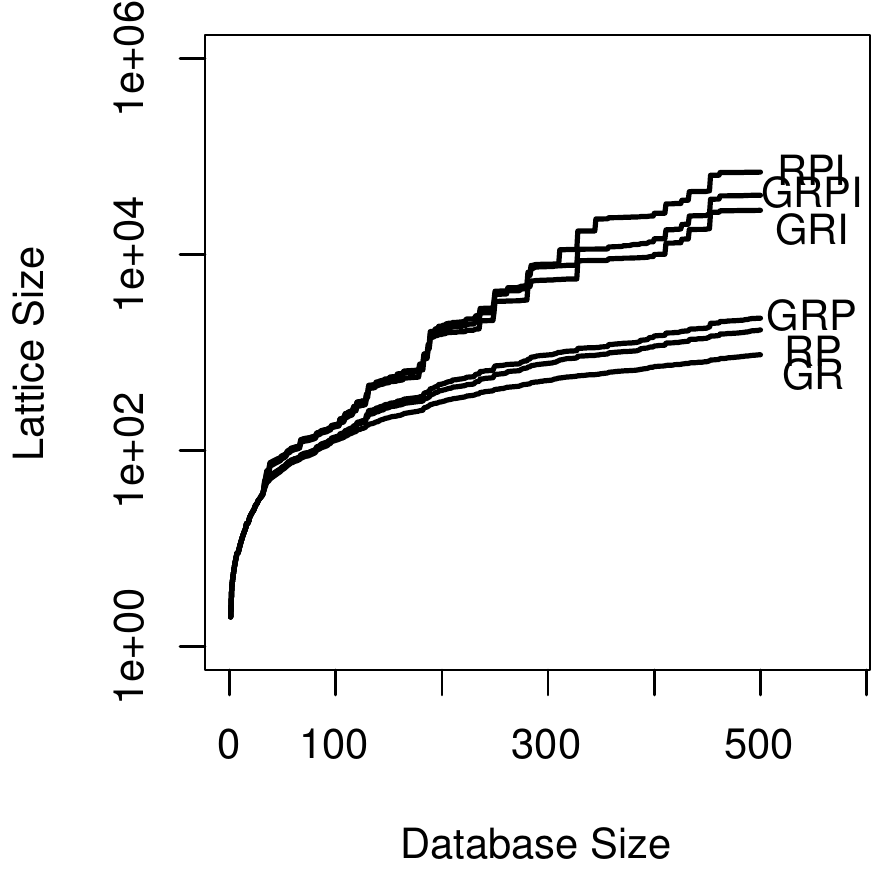}	
		\caption{MLP projection, $\ell=3$}
		\label{fig:lattice-size-stats-ell3}
	\end{subfigure}
	\caption{Lattice size for different projections}
	\label{fig:lattice-size-stats}
\end{figure}

Figure~\ref{fig:computational-stats} shows computational times for different projections as a function of dataset size. Figure~\ref{fig:computational-stats-ell2} shows different alphabet projections for MLP projection with $\ell=2$, while Figure~\ref{fig:computational-stats-ell3} for MLP with $\ell=3$. Every alphabet projection is given by the name of fields, that are considered within the projection: \texttt{G} corresponds to hospital geo-location, \texttt{R} is the reason for a hospitalization, \texttt{P} is medical procedures and \texttt{I} is repetition interval, i.e. the number of consequent hospitalizations with the same reason. We can see from these figures that MLP allows one to save some computational resources with increasing of $\ell$. The difference in computational time between $\ell=2$ and $\ell=3$ projections is significant, especially for time consuming cases. Even a bigger variation can be noticed for the alphabet projections. For example, computation of the \texttt{RPI} projection takes 100 times more resources than any from \texttt{GRP, RP, GR, GRP}.

The same dependency can be seen in Figure~\ref{fig:lattice-size-stats}, where the number of concepts for every projection is shown. Consequently, it is important for an expert to provide a strict projection that allows him to answer his questions in order to save computational time and memory. 

Table~\ref{tbl:result-concepts} shows some interesting concept intents with the corresponding support and ranking w.r.t. concept stability. For example the concept \#1 is obtained under the projection $GR$ (i.e., we consider only hospital and reason), with the intent $\seq{\elnP{Lorraine,C341~Lung\ Cancer}}$, where \texttt{C341 Lung Cancer} is a special kind of lung cancer (malignant neoplasm in Upper lobe, bronchus or lung). This concept is the most stable concept in the lattice for the given projection, and the size of the concept extent is $287$ patients.

\begin{table}
	\tbl{Interesting concepts, for different projections.}{
	\resizebox{\columnwidth}{!}{
		\begin{tabular}{rcccc}
			\toprule
			\# & Projection & Intent & Stab. Rank & Support \\
			\colrule
			$1$ & $GR$  & $\seq{\elnP{Lorraine,C341~Lung\ Cancer}}$ & $1$ & $287$ \\
			$2$ & $GR2$  & $\seq{\elnP{Lorraine,Respiratory\ Disease}\+\elnP{CHU_{Nancy}, Lung\ Cancer}}$ & $26$ & $22$ \\
			$3$ & $GR3$  & $\seq{\elnP{Lorraine,Chemotherapy} \times 4}$ & $1$ & $176$ \\
			$4$ & $RPI3$  & $\begin{aligned}\seq{\elnP{Preparation~for~Chemotherapy,\{Lung~\f&\s{}Radiography\}}\seqdlmt\f \s\elnP{Chemotherapy} \times [3,4]}\end{aligned}$  & $5$ & $36$ \\
			\botrule
		\end{tabular}
	}}
	\label{tbl:result-concepts}
\end{table}

One of the questions that the analyst would like to address here is \textit{``Where do patients stay (i.e. hospital location) during their treatment, and for which reason~?''}. To answer this question, we consider only healthcare institutions and reason fields, requiring both to ``hold'' some information and we use the MLP projection of length 2 and 3 (i.e. projections $GR2$ and $GR3$). Nearly all frequent trajectories show that patients usually are treated in the same region. However, \emph{pattern \#2} obtained under $GR2$ projection shows that, \textit{``22 patients were first admitted in some healthcare institution in Lorraine region for a problem related to the respiratory system and then they were treated for a lung cancer in University Hospital of Nancy.''} 

Another interesting question  is \textit{``What are the sequential relations between hospitalization reasons and the corresponding procedures?''}. 
To answer this question, we are not interested in healthcare institutions. Thus, any alphabet element is projected by substituting healthcare institution field with `*'. As hospitalization reason is important in each hospitalization, any alphabet element without the hospitalization reason is of no use and is projected to the bottom element $\bote$ of the alphabet. Such projections are called $RPI2$ or $RPI3$, meaning that we consider the fields ``Reason'' and ``Procedures'', while the reason should not be empty and the MLP parameter is $2$ or $3$. \emph{Pattern \#4} trivially states that, \textit{``36 patients with lung cancer are hospitalized once for the preparation of chemotherapy and during this hospitalization they undergo lung radiography.  Afterwards, they are hospitalized between 3 and 4 times for chemotherapy.''}

%% NICO
%% Added a partial conclusion about the interest of mining SPS
Variability is high in healthcare processes and affects many aspects of healthcare trajectories: patients, medical habits and protocols, healthcare organisation, availability of treatments and settings\ldots{} Mining sequential pattern structures is an interesting approach for finding regularities across one or several dimensions of medical trajectories in a population of patients. It is flexible enough to help healthcare managers to answer specific questions regarding the natural organisation of care processes and to further compare them with expected or desirable processes. The use of taxonomies plays also a key role in finding the right level of description of sequential patterns and reducing
the interpretation overhead.

\section{Related work}\label{sect:related-work}

\citet{Agrawal1995} introduced the problem of mining sequential patterns over large sequential databases. Formally, given a set of sequences, where each sequence is a list of transactions ordered by time and each transaction is a set of items, the problem amounts to find all frequent subsequences that appear a sufficient number of times with a user-specified minimum support threshold (\emph{minsup}). Following the work of Agrawal and Srikant many studies have contributed to the efficient mining of sequential patterns \citep{Mooney2013}. Most of them are based on the antimonotonicity property (used in \emph{Apriori}), which states that any super pattern of a non-frequent pattern cannot be frequent. The main algorithms are PrefixSpan \citep{PeiHPCDH01}, SPADE \citep{Zaki:2001}, SPAM \citep{AyresFGY02}, PSP \citep{MassegliaCP98}, DISC \citep{ChiuWC04}, PAID \citep{YangKW06} and FAST \citep{FAST}. All these algorithms aim at discovering sequential patterns from a set of sequences of itemsets such as customers who frequently buy DVDs of episodes I, II and III of Stars Wars, then buy within 6 months episodes IV, V, VI of the same famous epic space opera.

Many studies about sequential pattern discovery focus on single-dimensional sequences. However, in many situations, the database is multidimensional in the sense that items can be of different nature. For example, a consumer database can hold information such as article price, gender of the customer, location of the store and so on. \citet{PintoHPWCD01} proposed the first work for mining multidimensional sequential patterns. In this work, a {\it multidimensional sequential database} is defined as a schema $(ID,D_{1},...,D_{m},S)$, where $ID$ is a unique customer identifier, $D_{1},...,D_{m}$ are dimensions describing the data and S is the sequence of itemsets. A {\it multidimensional sequence} is defined as a vector $\langle \{d_{1},d_{2},...,d_{m}\} ,S_1,S_2,...,S_l \rangle$ where $ d_{i} \in D_{i}$ for ($i \leqslant m$) and  $S_1,S_2,...,S_l,$ are the itemsets of sequence $S$. For instance, $\langle \{Metz,Male\},\{mp_{1},$ $mp_{2}\},\{mp_{3}\} \rangle$ describes a male patient who underwent  procedures $mp_{1}$ and $mp_{2}$ in Metz and then underwent $mp_{3}$ also in Metz. Here, dimensions remain constant over time, such as the location of the treatment. This means that it is not possible to have a pattern indicating that when the patient underwent procedures $mp_{1}$ and $mp_{2}$ in Metz then he underwent $mp_{3}$ in Nancy. Among other proposals, \citet{Yu2005} proposed two methods AprioriMD and PrefixMDSpan for mining multidimensional sequential patterns in the web domain. This study considers pages, sessions and days as dimensions. Actually, these three different dimensions can be projected into a single dimension corresponding to web pages, gathering web pages visited during a same session and ordering sessions w.r.t the day as order.

In real world applications, each dimension can be represented at different levels of granularity, by using a poset. For example, apples in a market basket analysis can be either described as fruits, fresh food or food. The interest lies in the capacity of extracting more or less general/specific multidimensional sequential patterns and overcome problems of excessive granularity and low support. \citet{Srikant1996} proposed GSP which uses posets for extracting sequential patterns. The basic approach is based on replacing every item with all the ancestors in the poset and then the frequent sequences are generated. This approach is not scalable in a
multidimensional context because the size of the database becomes the product of maximum height of the posets and number of dimensions. 

\citet{Plantevit2010} defined a {\it multidimensional sequence} as an ordered list of multidimensional items, where a { \it multidimensional item} is a tuple $(d_1,...,d_m)$ and $d_{i}$ is an item associated with the $i^{th}$ dimension. They proposed $M^{3}SP$, an approach taking both aspects into account where each dimension is represented at different levels of granularity, by using a poset. $M^{3}SP$ is able to search for sequential patterns with
the most appropriate level of granularity. Their approach is based  on the extraction of the most specific frequent multidimensional items, which are then used as alphabet to rephrase the original database. Then, $M^{3}SP$ uses a standard sequential pattern mining algorithm to extract multidimensional sequential patterns. However, $M^{3}SP$ is not adapted to mine sequential databases,  where sequences are defined over a combination of sets of items and items lying in a poset. Then it is not possible to have a pattern indicating that when the patient went to $uh_{p}$ for a problem of cancer $ca$, where he underwent procedures $mp_{1}$ and $mp_{2}$, then he went to $gh_{l}$ for the same medical problem $ca$, where he underwent $mp_{3}$ ( i.e, $\langle (uh_{p},ca,\{mp_1,mp_2\}),(gh_{l},ca,\{mp_3\})\rangle $).
Our approach allows us to process such kind of patterns and in addition the elements of sequences are even more general. For example, beside multidimensional and multilevel sequences, sequences of graphs fall under our definition. Moreover, frequent subsequence mining gives rise to a lot of subsequences which can be hardly analyzed by an expert. Since our approach is based on Formal Concept Analysis (FCA)~\citep{Ganter1999}, we can use efficient relevance indexes defined in FCA.

This paper is not the first attempt to use FCA for the analysis of sequential data. \citet{Ferre2007} processes sequential datasets based on a ``simple'' alphabet without involving any partial order.  
%in this approach maximal common subsequences (with no gaps) were mined and analyzed with FCA.
In~\citet{Casas-Garriga2005} only sequences of itemsets are considered. All closed subsequences are firstly mined and then regrouped by a specialized algorithm in order to obtain a lattice similar to the FCA lattice. This approach was not verified experimentally. Moreover, compared with both approaches, i.e. \citet{Ferre2007} and \citet{Casas-Garriga2005}, our approach suggests a more general definition of sequences and, thanks to pattern structures, there is no `pre-mining' step to find frequent (or maximal) subsequences. This allows us to apply different ``projections'' specializing the request of an expert and simplifying the computations. In addition, in our approach nearly all state-of-the-art FCA algorithms can be used in order to efficiently process a dataset.  
%\citet{Garriga2012} process multirelational databases by extending LCM~\citep{Uno2004}, which is closely related to FCA. Although this approach is efficient for special kinds of multirelational databases, it cannot process sequential and graph datasets for reasons explained in their paper.

There is a number of approaches that help to analyze medical treatment data. However, the direct comparison of them is hardly possible, because every approach is designed for its own problem. For example, \citep{Tsumoto2014} analyze data of one hospital and provide a different view on the processes within the hospital w.r.t. our approach. Finally and naturally, the most similar approach to our work can be found in~\citep{Egho2014,Egho2014a}, as some authors of the present paper are involved in this alternative work. In~\citep{Egho2014,Egho2014a}, authors mine frequent sequences of the dataset similar to the sequences studied here. However, they approach the complexity of the analysis of such data in a different way. They use a support threshold in order to specify the outcome of the algorithm and do not provide any order in which one can analyze the result. In our case we rely on projections that are usually simpler to incorporate expert knowledge than a support threshold and we give an order (w.r.t. stability of a concept) which can be used to simplify the analysis of the treatment data.

%Projections is an essential part of our approach and can be considered as a special kind of filtering with mathematical properties. On the other hand many constraints that do not change subsequence relation have a corresponding projection. In their survey~\citet{Mooney2013} enumerate 8 types of constraints (Section 5), two of them, i.e. ``item constraint'' and ``length constraint'', correspond to the projections defined in this work.

%
\section{Conclusion}\label{sect:conclusion}

In this paper, we have presented a novel approach for analyzing sequential data within the framework of pattern structures, an extension of Formal Concept Analysis dealing with complex data.
It is based on the formalism of sequential pattern structures and projections. Our work complements the general orientations towards \emph{statistically significant patterns} by presenting strong formal results on the notion of interestingness from a concept lattice viewpoint. 
% The classical formalism of strings is adapted to deal with sequences in an original way.
The framework of pattern structures is very flexible and shows some important properties, for example in allowing to reuse state-of-the-art and efficient FCA algorithms.
Using pattern structures leads to the construction of a pattern concept lattice, which does not require the setting of a support threshold, as usually needed in classical sequential pattern mining.
Moreover, the use of projections gives a lot of flexibility especially for mining and interpreting special kinds of patterns (patterns can be proposed at several levels of complexity w.r.t. extraction and interpretation).

Our framework was tested on a real-world dataset with patient hospitalization trajectories.
Interesting patterns answering questions of an expert are extracted and interpreted, showing the feasibility and usefulness of the approach, and the importance of the stability as a pattern-selection procedure.
In particular, projections play an important role here: mainly, they provide means to select patterns of a special interest and they help to save computational time (which could be otherwise very large).

For future work, we are planning to more deeply investigate projections, their potential w.r.t. the types of patterns. It can be interesting to introduce and evaluate the stability measure directly on sequences. Another research direction is mining of association rules or building a Horn approximation~\citep{Balcazar2005} from the stable part of the pattern lattice or stable sequences. Finally, as discussed above, a precise study combining frequent subsequence mining and FCA-based approaches should be carried out.
% (such a study was never completely carried out).
%Finally, another research direction is to unify this framework within a general approach based on pattern structures and projections, able to deal with strings, sequences and graphs.
% The feasibility of such an approach remains an open question but some answers could be provided with adapted kinds of projections.

%
%\input{sections/appendix}
%
\section*{Acknowledgments}
The fourth co-author was supported within the framework of the Basic Research Program at National Research University Higher School of Economics (Moscow).

\section*{Notes on contributors}
{\small
	\begin{description}
		\item[Aleksey Buzmakov]
			is a PhD student in Informatics at Universit\'{e} de Lorraine (Van\-doevre les Nance, France). He holds master and bachelor degree in applied mathematics and physics from Moscow Institute of Physics and Technology. His research interest includes data mining and artificial intelligences. In particular he works with Formal Concept Analysis and Pattern Structure in order to mine complex data such as sequences or graphs.
		\item[Elias Egho]
			is a Post Doctoral Researcher in Orange Labs (France Telecom Research and Development) with Profiling \& Data Mining team. In 2014, he received a PhD degree in Computer Science from University of Lorraine, Nancy, France in LORIA-INRIA Nancy Grand Est laboratory. His main research interest is mining sequential patterns for detection and classification of sequential data.
		\item[Nicolas Jay]
			is a professor of biostatistics and medical informatics at the Université de Lorraine. His research interests include medical knowledge representation and knowledge discovery in medical databases, with applications to patient trajectory analysis. He works as a public health physician at the University Hospital of Nancy.
	\item[Sergei O. Kuznetsov]
			is a professor of the National Research University Higher School of Economics (HSE), Moscow, where he is the head of department of data analysis and artificial intelligence. He defended habilitation thesis (``Doctor of Science'') at the Computer Center of the Russian Academy of Sciences (Moscow, Russia) in 2002. He holds the ``Candidate of Science'' degree (PhD equivalent) from VINITI (Moscow, Russia) since 1990. His research interests include mathematical models, algorithms and algorithmic problems of machine learning, formal concept analysis, data mining, and knowledge discovery.
		\item[Amedeo Napoli]
			is a CNRS senior scientist (DR CNRS) and the scientific leader of the Orpailleur research team at LORIA/Inria Laboratory in Nancy. His scientific interests are knowledge discovery (pattern mining and Formal Concept Analysis) and knowledge representation (ontology engineering). He is involved in many national and international research projects with applications in agronomy, biology, chemistry, and medicine.
		\item[Chedy Raïsi] 
			received his PhD in Computer Science from the University of Montpellier and the Ecole des Mines d'Alès in July 2008. He is currently a research scientist (``Charg\'{e} de recherche 1'') at the Institut "National de Recherche en Informatique et en Automatique" (INRIA) in France. His research interests includes pattern mining and privacy-preserving data analysis.
	\end{description}
}

\bibliographystyle{plainnat}
\putBibliography

\end{document}